\documentclass[letterpaper]{article} % DO NOT CHANGE THIS
\usepackage{aaai}  % DO NOT CHANGE THIS
\usepackage{times}  % DO NOT CHANGE THIS
\usepackage{helvet}  % DO NOT CHANGE THIS
\usepackage{courier}  % DO NOT CHANGE THIS
\usepackage[hyphens]{url}  % DO NOT CHANGE THIS
\usepackage{graphicx} % DO NOT CHANGE THIS
\usepackage{natbib}  % DO NOT CHANGE THIS AND DO NOT ADD ANY OPTIONS TO IT
\usepackage{caption} % DO NOT CHANGE THIS AND DO NOT ADD ANY OPTIONS TO IT
\usepackage{algorithm}
\usepackage{algorithmic}

\usepackage{newfloat}
\usepackage{listings}
\DeclareCaptionStyle{ruled}{labelfont=normalfont,labelsep=colon,strut=off} % DO NOT CHANGE THIS
\floatstyle{ruled}
\newfloat{listing}{tb}{lst}{}
\floatname{listing}{Listing}
\usepackage{booktabs}
\usepackage{amsthm}
\usepackage{amsmath}
\usepackage{amssymb}
\usepackage{bm}
\usepackage{tikz}
\usetikzlibrary{shapes.geometric}
\usetikzlibrary{arrows.meta}

\newtheorem{theorem}{Theorem}
\newtheorem{lemma}[theorem]{Lemma}
\newtheorem{definition}{Definition}
\newtheorem{assumption}{Assumption}
\newcommand{\astar}{A$^\ast$}

\title{Domain Knowledge in A*-Based Causal Discovery}
\author {
    Steven Kleinegesse,\textsuperscript{\rm 1}
    Andrew R. Lawrence,\textsuperscript{\rm 1}
    Hana Chockler\textsuperscript{\rm 1,\rm 2}
}
\affiliations{
    \textsuperscript{\rm 1} causaLens\\
    \textsuperscript{\rm 2} Department of Informatics, King's College London \\
    \{steven, andrew, hana\}@causalens.com
}

\begin{document}

\maketitle

\begin{abstract}
Causal discovery has become a vital tool for scientists and practitioners wanting to discover causal relationships from observational data. While most previous approaches to causal discovery have implicitly assumed that no expert domain knowledge is available, practitioners can often provide such domain knowledge from prior experience. Recent work has incorporated domain knowledge into constraint-based causal discovery. The majority of such constraint-based methods, however, assume causal faithfulness, which has been shown to be frequently violated in practice. Consequently, there has been renewed attention towards exact-search score-based causal discovery methods, which do not assume causal faithfulness, such as A*-based methods. However, there has been no consideration of these methods in the context of domain knowledge. In this work, we focus on efficiently integrating several types of domain knowledge into A*-based causal discovery. In doing so, we discuss and explain how domain knowledge can reduce the graph search space and then provide an analysis of the potential computational gains. We support these findings with experiments on synthetic and real data, showing that even small amounts of domain knowledge can dramatically speed up A*-based causal discovery and improve its performance and practicality.
\end{abstract}

\section{Introduction}

Discovering causal relationships is of key scientific and practical importance, with wide-ranging applications in molecular biology~\citep{triantafillou2017}, climate science~\citep{ebert2014}, marketing~\citep{Orriols2010} and healthcare~\citep{mani2000}, among many others. Interventions or randomized controlled trials are often the most effective and reliable means of identifying causal relationships, but they tend to be challenging, expensive and in some cases unethical. Causal discovery has emerged to tackle this problem by only requiring observational data and has received considerable attention in the recent decades~\citep[see][for a recent review]{glymour2019}.

Most approaches to causal discovery assume that no prior knowledge about the underlying causal graph is available, meaning that all causal relationships are equally plausible \emph{a priori}. In practice, however, domain experts often have information about the causal relationships between subsets of variables. They may know from prior experimentation and/or domain expertise that, e.g., two particular variables directly causally influence each other, certain pairs of variables cannot be causally connected, or that there is a causal hierarchy of variables. 
% As an example, it is generally accepted that variables such as age cannot be influenced by anything else, i.e., they should be source variables wherever they appear. 
Incorporating this domain knowledge could dramatically reduce the search space of graphs in causal discovery~\citep{constantinou2021}.
%, but has received limited attention so far.
% contributes little to the overall research direction of the causality community.

Recently,~\citet{andrews2020} have proposed a constraint-based causal discovery method that integrates domain knowledge, specifically a hierarchy of nodes, into the search process. Their method, which extends Fast Causal Inference (FCI)~\citep{spirtes1995, zhang2008}, does not assume causal sufficiency, allowing it to detect the presence of latent confounders that are not present in the data. However, being a constraint-based method, the method of~\citet{andrews2020} does assume causal faithfulness, which means that any conditional independencies observed in the data must also be present in the discovered causal graph. This assumption was shown to be frequently violated in practice~\citep{uhler2013} when the number of observations in the data is limited. Furthermore, causal discovery methods based on FCI can scale exponentially with the number of variables in the data~\citep{glymour2019} and rely on expensive conditional independence (CI) tests.
%that have to be applied to a combinatorial number of conditioning sets.

There has been renewed interest in score-based causal discovery methods, which, unlike constraint-based approaches, do not rely on CI tests. While the usefulness of these methods has been limited for a long time, mainly due to poor scalability and computational efficiency, recent score-based methods such as FGES~\citep{ramsey2017}, SE-GES~\citep{chickering2020} and {\astar}-based methods~\citep{yuan2013, lu2021, ng2021} have been shown to be efficient and scalable even for hundreds of variables. To our knowledge, however, there has been no work that efficiently integrates domain knowledge into score-based causal discovery methods that do not assume causal faithfulness, such as those based on \astar. 

In this work, we focus on the efficient and sound integration of domain knowledge into {\astar}-based causal discovery algorithms. 
%, which, to the best of our knowledge, has not been done before. 
% We here focus on {\astar}-based methods as an important exemplar of recent exact-search score-based causal discovery algorithms. 
We believe that our approach improves applicability of score-based causal discovery methods in general and helps to bridge the gap between theoretical approaches and the needs of practitioners. In summary, the contributions of this work are as follows:
\begin{enumerate}
     \item Extending {\astar}-based causal discovery methods to efficiently integrate several types of domain knowledge;
     \item Providing a theoretical analysis of the computational gains obtained when incorporating domain knowledge;
     \item Experimental results demonstrating the computational gains and performance improvements in practice.
     % \item Showcasing the empirical gains obtained when integrating different types of domain knowledge into {\astar}-based causal discovery.
\end{enumerate}

\section{Background}

Let $\mathcal{G} = (\mathbf{V}, \mathbf{E})$ be a directed acyclic graph (DAG) with vertices $\mathbf{V} = \{x_1, \dots, x_d\}$ corresponding to random variables. The set of edges $\mathbf{E}$ contains tuples $(x_i, x_j) \in \mathbf{V}^2$ that specify directed edges going from nodes $x_i \in \text{Pa}_{\mathcal{G}}(x_j)$ to nodes $x_j$, where $\text{Pa}_{\mathcal{G}}(x_j) \subseteq \mathbf{V} \backslash x_j$ is the parent set of node $x_j$ in $\mathcal{G}$. The general aim of causal discovery is to use observational data to identify the correct parent set $\text{Pa}_{\mathcal{G}}(x_i)$ for each node $x_i$, while maintaining acyclicity, thereby fully recovering the true DAG $\mathcal{G}$.

\subsection{Common assumptions in causal discovery}

Causal discovery methods typically rely on a set of common assumptions,
% that are often needed to recover the true DAG from observational data
which we briefly discuss below.

\begin{assumption}[Markov]
Consider a DAG $\mathcal{G}$ with vertices $\mathbf{V}$ and a probability distribution $\mathbb{P}$ over the random variables contained in $\mathbf{V}$. $\mathcal{G}$ and $\mathbb{P}$ satisfy the Markov assumption if and only if every variable $x_i \in \mathbf{V}$ is conditionally independent of its non-descendants $\mathbf{V} \backslash \{ \text{descendants}_{\mathcal{G}}(x_i) \cup \text{Pa}_{\mathcal{G}}(x_i) \cup x_i \}$ given its parents $\text{Pa}_{\mathcal{G}}(x_i)$.
\end{assumption}

Intuitively, the above Markov assumption says that any conditional independence relationships entailed by the DAG $\mathcal{G}$ also need to be present in the probability distribution $\mathbb{P}$. There are usually several DAGs that entail the same conditional independence relations, which are therefore known as members of the same Markov equivalence class (MEC). An MEC is uniquely determined by its skeleton and its v-structures~\citep{pearl2009}, i.e.~collider connections of the form $X \rightarrow Y \leftarrow Z$ where $X$ and $Z$ are not connected (also known as unshielded colliders). An important type of MEC is a completed partially directed acyclic graph (CPDAG), which contains both undirected and directed edges that define its skeleton and v-structures~\citep{meek1995}.

\begin{assumption}[Faithfulness] Consider a DAG $\mathcal{G}$ with vertices $\mathbf{V}$ and a probability distribution $\mathbb{P}$ over the random variables contained in $\mathbf{V}$. $\mathcal{G}$ and $\mathbb{P}$ satisfy the Faithfulness assumption if and only if $\mathbb{P}$ does not imply any conditional independence relations that are not already entailed by the Markov assumption.
\end{assumption}

The faithfulness assumption above says that any conditional independence relationships entailed by the the probability distribution $\mathbb{P}$ also need to be present in the DAG $\mathcal{G}$, which can be understood as the reverse of the Markov assumption. Due to errors in conditional independence testing, typically occurring due to finite data, this assumption has been shown to often be approximately violated in practice~\citep{uhler2013}. 
% In fact, if the data is discrete this assumption may also be exactly violated for certain DAGs even when the available data is unlimited.

\begin{assumption}[Sufficiency] A DAG $\mathcal{G}$ satisfies the causal sufficiency assumption if there are no unmeasured common causes.
\end{assumption}

The above causal sufficiency assumption intuitively says that all common causes should be captured by the DAG, meaning that there should not be any unknown, or unmeasurable, latent confounders.

\subsection{Score-based causal discovery}

In general, score-based causal discovery methods optimize a score-function that defines the fit of candidate causal structures, e.g.~parent-children combinations, until we obtain a DAG that 
% is a directed acyclical graph (DAG) is obtained that 
is optimal with respect to the particular score function used~\citep[see][for a recent review]{glymour2019}. Commonly-used score functions are the Bayesian information criterion (BIC) or minimum description length (MDL) for continuous data~\citep{schwarz1978}, the BGe score for linear Gaussian data~\citep{geiger1994} and the BDeu score for discrete data~\citep{heckerman1995}. 
% The optimal score function depends on the particular use-case, what type of data we are dealing with, and how many assumptions we are willing to make. 
% Most score functions make assumptions about the functional relationships between parents and their children, with linearity being one of the most common ones, or they make assumptions about the distribution of the noise variables, with Gaussianity being relatively common.

We follow previous work on {\astar}-based causal discovery~\citep{yuan2013, lu2021, ng2021} and utilize the BIC score function, thereby assuming linearity and Gaussianity.

\begin{definition}[BIC in Causal Discovery]
Assume that the relationship between a node $x_i$ and its parents $\text{Pa}_{\mathcal{G}}(x_i)$ is modeled by a parametrized function $f_{\bm{\theta}}(\text{Pa}_{\mathcal{G}}(x_i))$ that is trained in a regression task, where $\bm{\theta}$ are some model parameters. The BIC score for node $x_i$ is given by
\begin{equation}
    S_{\text{BIC}}(x_i) = k \log{N} - 2 \log{L_{\text{max}}},
\end{equation}
where $k$ is the number of parameters in the functional relationship between parents and children, $N$ is the number of observations in the data and $L_{\text{max}}$ is the value of the maximized likelihood function after training. Under the assumptions of linearity and Gaussianity, the BIC score function can be approximated as follows,
\begin{equation}
    S_{\text{BIC}}(x_i) \approx k \log{N} + N \log{R / N},
\end{equation}
where $R$ is the sum of squared residuals from ordinary least squares regression.
\end{definition}

A prominent example of a score-based causal discovery method is greedy equivalence search (GES)~\citep{chickering2003}.
% , which starts with an empty graph and consists of two stages: first adding edges and then removing edges until the overall value of the score function, e.g.~BIC, can no longer be improved. 
While GES works well in practice, it is a greedy method and therefore only guaranteed to converge to a local optimum. Moreover, GES relies on all of the aforementioned assumptions, i.e.~the Markov, faithfulness and sufficiency assumptions.

\subsection{{\astar}-based algorithms}

Recently, score-based causal discovery methods that are based on {\astar}~\citep{hart1968}, which is a shortest path-finding algorithm, have gained traction. {\astar}-based methods are \emph{exact-search} score-based causal discovery algorithms, meaning that they exhaustively explore the graph search space and are assured to converge to a global optimum, unlike GES. While {\astar}-based methods also rely on the Markov and sufficiency assumptions, they do not rely on the faithfulness assumption. Instead, assuming that the BIC score function is used, exact-search methods can be shown to rely on the weaker sparsest Markov representation (SMR) assumption~\citep{lu2021, ng2021}.

\begin{assumption}[Sparsest Markov Representation] Consider a DAG $\mathcal{G}$ with vertices $\mathbf{V}$ and a probability distribution $\mathbb{P}$ over the random variables contained in $\mathbf{V}$. The Markov equivalence class (MEC) of $\mathcal{G}$ is the unique sparsest MEC that satisfies the Markov assumption with $\mathbb{P}$.
\end{assumption}

The above assumption is also referred to as the (unique) frugality assumption and has various desirable properties~\citep{forster2018}. Importantly, it has been shown that causal discovery methods that rely on the faithfulness assumption can perform worse than those that rely on the SMR assumption~\citep{lu2021, ng2021}.

{\astar}-based causal discovery methods generally work by first constructing a set of parent graphs for each node, as explained below. These are then traversed efficiently using {\astar} in order to find the optimal parent set for each node.

\subsubsection{Constructing parent graphs} For a given node $x_i$, the state space of candidate parent sets $\text{Pa}_{\mathcal{G}}(x_i)$ can be represented by a \emph{parent graph}, which we denote by $\mathcal{G}_{Pa}(x_i)$. Concretely, this parent graph is a Hasse diagram, or order graph, consisting of all combinatorial subsets of variables in $\mathbf{V} \backslash x_i$~\citep[see][]{yuan2013}. Figure~\ref{fig:parent_graph_example} provides an example parent graph for one variable with three potential parents. Parent graphs are constructed in a top-down approach, with the top node being the empty set $\varnothing$ and the bottom node being all other nodes $\mathbf{V} \backslash x_i$. Importantly, each node in the parent graph represents one potential parent set that has a certain score associated with it, measuring its quality as a parent set, as provided by a score-function such as BIC. The cost of a path from the top to the bottom node is simply the sum of the scores from all traversed nodes. In this context, the shortest path in $\mathcal{G}_{Pa}(x_i)$ is the top-to-bottom path with the minimum total cost. Once we have access to the shortest path for each parent graph $\mathcal{G}_{Pa}(x_1), \dots, \mathcal{G}_{Pa}(x_d)$, we can use these to construct a DAG that is optimal by construction~\citep[see][for more information]{yuan2013}.

%\begin{figure}[t!]
%    \centering
%    \includegraphics[width=0.95\linewidth]{images/parent_graph_example_new.pdf}
%    \caption{Illustration of a parent graph for node $x_1$ with potential parents $\{x_2, x_3, %x_4\}$.}
%    \label{fig:parent_graph_example}
%\end{figure}

\subsubsection{Traversing parent graphs} As mentioned above, finding the shortest path for each parent graph allows us to construct an optimal DAG.~\citet{yuan2013} proposed to use the {\astar} path-finding algorithm to find this shortest path. One of the main advantages of {\astar} is that it uses a heuristic function that estimates the quality of yet unexplored paths, which means that it does not need to explore the whole search space in order to find the optimal solution. Although this has a higher memory-cost, it allows {\astar} to scale up to larger dimensions by reducing the search space. We refer the reader to the work of~\citet{yuan2013} for more detailed explanations of the regular {\astar}-based causal discovery method.

\subsubsection{Variations}

Various extensions to the above {\astar}-based causal discovery have been proposed, mostly aiming to increase computational efficiency and scalability. The earliest of these extensions was Triplet {\astar}~\citep{lu2021}, which runs the above {\astar}-based algorithm on triplets of variables, combines the results and then resolves conflicts between them using a set of rules. However, this extension does not represent an exact-search causal discovery method and, more importantly, it relies on constraint-based approaches, such as PC~\citep{spirtes1993} or MMMB~\citep{tsamardinos2003}, to identify appropriate triplets, thereby assuming faithfulness. These limitations were rectified by~\citet{ng2021}, who proposed {\astar}-SuperStructure and Local {\astar}. Both of these variations rely on a so-called \emph{super-structure}, which is an undirected graph that is a super-set of the skeleton of the true DAG, i.e.~the true skeleton with some spurious edges.~\citet{ng2021} use graphical LASSO~\citep{friedman2008} to estimate this super-structure and show that this estimation method relies on assumptions strictly weaker than faithfulness. Importantly, the super-structure, often represented by a binary adjacency matrix, specifies forbidden causal relationships which can be used to prune the parent graphs and reduce the search space accordingly. Local {\astar} further improves scalability by looping through local clusters of nodes, as defined by the super-structure, and then applying {\astar}-SuperStructure to each of them, combining the results and resolving any conflicts using a set of rules.

\begin{figure}[t!]
\centering
\begin{tikzpicture}[
    every node/.append style={draw, ellipse, minimum width = 2cm, minimum height = 0.8cm, semithick}]
    \node[fill=gray!10] (empty) at (2.5,0) {\Large $\varnothing$};
    \node[fill=gray!10] (x2) at (0,-1.25) {$\{x_2\}$};
    \node[fill=gray!10] (x3) at (2.5,-1.25) {$\{x_3\}$};
    \node[fill=gray!10] (x4) at (5,-1.25) {$\{x_4\}$};
    \node[fill=gray!10] (x2x3) at (0,-2.5) {$\{x_2, x_3\}$};
    \node[fill=gray!10] (x2x4) at (2.5,-2.5) {$\{x_2, x_4\}$};
    \node[fill=gray!10] (x3x4) at (5,-2.5) {$\{x_3, x_4\}$};
    \node[fill=gray!10] (x2x3x4) at (2.5,-3.75) {$\{x_2, x_3, x_4\}$};

    \path[-latex, semithick]
    	(empty) edge (x2)
    	(empty) edge (x3)
    	(empty) edge (x4)
    	(x2) edge (x2x3)
    	(x2) edge (x2x4)
    	(x3) edge (x2x3)
    	(x3) edge (x3x4)
    	(x4) edge (x2x4)
    	(x4) edge (x3x4)
    	(x2x3) edge (x2x3x4)
    	(x2x4) edge (x2x3x4)
    	(x3x4) edge (x2x3x4);
\end{tikzpicture}
    \caption{Full parent graph $\mathcal{G}_{Pa}(x_1)$ for node $x_1$ with potential parents $x_2$, $x_3$ and $x_4$.}
    % Each of the shown nodes requires computation of a BIC score, measuring their quality as a parent set of $x_1$. {\astar} then finds the shortest weighted path from the top node to the bottom node, which is the path that minimizes the total cost of traversed nodes.}
    \label{fig:parent_graph_example}
\end{figure}
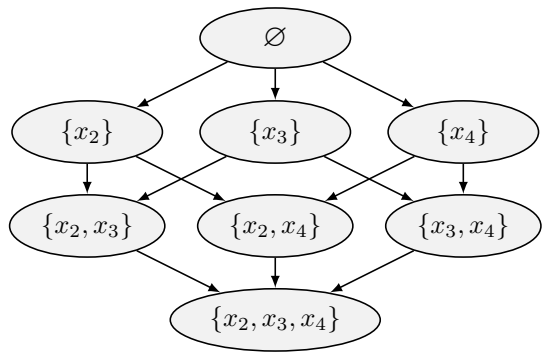

\section{Methodology}

We define different types of domain knowledge and discuss how to integrate these efficiently with {\astar}-based causal discovery methods.
% exact-search score-based causal discovery methods based on {\astar}. 
We then provide a theoretical analysis of the potential computational gains obtained.
%when inserting different types of domain knowledge.

\subsection{Domain knowledge}

\begin{figure*}
\centering
\begin{minipage}{.47\linewidth}
\centering
\begin{tikzpicture}[
    every node/.append style={draw, ellipse, minimum width = 2cm, minimum height = 0.8cm, semithick}]
    \node[dotted] (empty) at (2.5,0) {\Large $\varnothing$};
    \node[dotted] (x2) at (0,-1.25) {$\{x_2\}$};
    \node[dotted] (x3) at (2.5,-1.25) {$\{x_3\}$};
    \node[fill=gray!10] (x4) at (5, -1.25) {$\{x_4\}$};
    \node[dotted] (x2x3) at (0,-2.5) {$\{x_2, x_3\}$};
    \node[fill=gray!10] (x2x4) at (2.5,-2.5) {$\{x_2, x_4\}$};
    \node[fill=gray!10] (x3x4) at (5,-2.5) {$\{x_3, x_4\}$};
    \node[fill=gray!10] (x2x3x4) at (2.5,-3.75) {$\{x_2, x_3, x_4\}$};

    \path[-latex, semithick]
    	(empty) edge[dotted, ->] (x2)
    	(empty) edge[dotted, ->] (x3)
    	(empty) edge[dotted, ->] (x4)
    	(x2) edge[dotted, ->] (x2x3)
    	(x2) edge[dotted, ->] (x2x4)
    	(x3) edge[dotted, ->] (x2x3)
    	(x3) edge[dotted, ->] (x3x4)
    	(x4) edge (x2x4)
    	(x4) edge (x3x4)
    	(x2x3) edge[dotted, ->] (x2x3x4)
    	(x2x4) edge (x2x3x4)
    	(x3x4) edge (x2x3x4);
\end{tikzpicture}
\end{minipage} \hfill
\begin{minipage}{0.47\linewidth}
\centering
\begin{tikzpicture}[
    every node/.append style={draw, ellipse, minimum width = 2cm, minimum height = 0.8cm, semithick}]
    \node[fill=gray!10] (empty) at (2.5,0) {\Large $\varnothing$};
    \node[fill=gray!10] (x2) at (0,-1.25) {$\{x_2\}$};
    \node[fill=gray!10] (x3) at (2.5,-1.25) {$\{x_3\}$};
    \node[dotted] (x4) at (5,-1.25) {$\{x_4\}$};
    \node[fill=gray!10] (x2x3) at (0,-2.5) {$\{x_2, x_3\}$};
    \node[dotted] (x2x4) at (2.5,-2.5) {$\{x_2, x_4\}$};
    \node[dotted] (x3x4) at (5,-2.5) {$\{x_3, x_4\}$};
    \node[dotted] (x2x3x4) at (2.5,-3.75) {$\{x_2, x_3, x_4\}$};

    \path[-latex, semithick]
    	(empty) edge[] (x2)
    	(empty) edge[] (x3)
    	(empty) edge[dotted, ->] (x4)
    	(x2) edge[] (x2x3)
    	(x2) edge[dotted, ->] (x2x4)
    	(x3) edge[] (x2x3)
    	(x3) edge[dotted, ->] (x3x4)
    	(x4) edge[dotted, ->] (x2x4)
    	(x4) edge[dotted, ->] (x3x4)
    	(x2x3) edge[dotted, ->] (x2x3x4)
    	(x2x4) edge[dotted, ->] (x2x3x4)
    	(x3x4) edge[dotted, ->] (x2x3x4);
\end{tikzpicture}
\end{minipage}
    \caption{Pruned parent graph $\mathcal{G}_{Pa}(x_1)$ for variable $x_1$ with potential parents $x_2$, $x_3$ and $x_4$ when incorporating a known edge $x_4 \rightarrow x_1$ (left) and when incorporating a forbidden edge between $x_1$ and $x_4$ (right).}
    \label{fig:parent_graph_pruned}
\end{figure*}
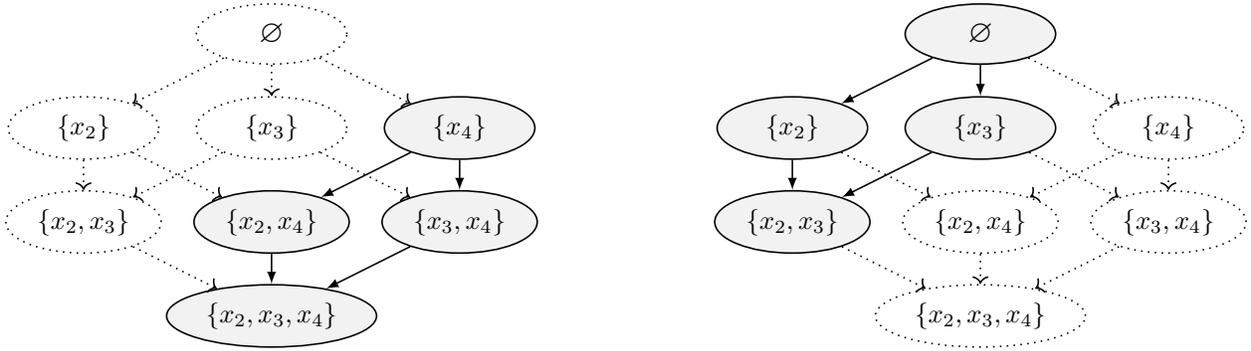

Many approaches to causal discovery assume that all potential causal structures are equally likely \emph{a priori}~\citep[e.g.][]{chickering2003, yuan2013, lu2021}. In practice, however,
% this may not always be the case. Instead, 
domain experts typically have useful prior experience that they can leverage to formulate domain knowledge about certain causal relationships. It may then be beneficial to integrate this domain knowledge into the causal discovery procedure in order to improve computational efficiency and accuracy. For instance, it is generally accepted that a person's age is a causal driver of their physical height, and not the other way around.
%, which means that we should be able to safely encode this known relationship into any relevant causal discovery procedure.

There are several ways that domain knowledge can be expressed and, conversely, encoded to be useful in causal discovery~\citep{constantinou2021}. In this work we focus on the following three common ways: 

\begin{definition}[Known edges]
Domain knowledge that encodes a known, directed edge $x_i \rightarrow x_j$ between two variables $x_i, x_j \in \mathbf{V}$ implies that $x_i$ is most definitely a member of the parent set of $x_j$, i.e.~$x_i \in \text{Pa}_{\mathcal{G}}(x_j)$. 
\end{definition}

\begin{definition}[Forbidden edges]
Domain knowledge that encodes a forbidden edge between two variables $x_i, x_j \in \mathbf{V}$ implies that neither $x_i$ nor $x_j$ can be a member of each other's parent set, i.e.~$x_i \not\in \text{Pa}_{\mathcal{G}}(x_j)$ and $x_j \not\in \text{Pa}_{\mathcal{G}}(x_i)$.
\end{definition}

\begin{definition}[Tiers]
Domain knowledge that encodes tiers $\mathbf{T} = \{\mathbf{T}_1, \dots, \mathbf{T}_n \}$ specifies a causal hierarchy of variables, which states that variables in later tiers, e.g., $\mathbf{T}_n$, cannot be causal drivers of variables in earlier tiers, e.g., $\mathbf{T}_1$. Specifically, if variable $x_i$ is a member of $\mathbf{T}_k$, variable $x_j$ is a member of $\mathbf{T}_s$ and $k < s$, the domain knowledge indicates that $x_i$ may be a parent of $x_j$, but $x_j$ cannot be a parent of $x_i$. If $x_i$ and $x_j$ are in the same tier, i.e.~$k=s$, then either variable may be a parent of the other.
\end{definition}

~\citet{meek1995} was among the first to consider domain knowledge that contains known edges and forbidden edges in constraint-based causal discovery. Similarly, \citet{scheines1998} were the first to insert tiered domain knowledge into constraint-based causal discovery. More recently,~\citet{andrews2020} inserted tiered domain knowledge into the popular FCI~\citep{spirtes1995, zhang2008} constraint-based algorithm. To our knowledge, however, the above ways of encoding domain knowledge have not been used in conjunction with exact-search score-based causal discovery methods.

We here make a few assumptions about the provided domain knowledge. 
First, we assume that there is \emph{no conflicting domain knowledge}, i.e.~we cannot specify a known edge $x_i \rightarrow x_j$ when we have already specified a forbidden edge between both variables. 
% For instance, we assume that we cannot specify that a variable $x_i \in \mathbf{V}$ is a direct causal driver of another variable $x_j \in \mathbf{V}$ when we have already specified that any edges between $x_i$ and $x_j$ should be forbidden. 
Second, we assume that the \emph{domain knowledge is correct}, i.e.~that any relationships encoded in the domain knowledge are also present in the ground-truth. Third, we assume that known and forbidden edges are examples of \emph{hard} domain knowledge, i.e.~the encoded domain knowledge needs to be present in the output of the algorithm. Tiers, on the other hand, is a mixture of \emph{soft and hard} domain knowledge, since variables in earlier tiers can potentially be parents of those in later tiers (soft), but those in later tiers cannot be parents of those in earlier tiers (hard).

% \begin{enumerate}
%     \item We assume that there is \emph{no conflicting domain knowledge}. For instance, we assume that we cannot specify that a variable $x_i \in \mathbf{V}$ is a direct causal driver of another variable $x_j \in \mathbf{V}$ when we have already specified that any edges between $x_i$ and $x_j$ should be forbidden. 
%     \item We assume that the \emph{domain knowledge is correct}, i.e.~that any relationships encoded in the domain knowledge are also present in the ground-truth DAG. \item We assume that known and forbidden edges are examples of \emph{hard} domain knowledge which must be upheld, i.e.~the encoded domain knowledge needs to be present in the output of the algorithm. Tiers, on the other hand, is a mixture of \emph{soft and hard} domain knowledge, since variables in lower tiers can potentially be parents of those in higher tiers (soft), but those in higher tiers cannot be parents of those in lower tiers (hard).
% \end{enumerate}

% \subsection{Inserting domain knowledge into {\astar}-based causal discovery}
\subsection{Integrating domain knowledge with {\astar}}

% The different types of domain knowledge described above have previously been integrated into constraint-based causal discovery before~\citet{meek1995, scheines1998, andrews2020}. We here focus on integrating domain knowledge into \emph{score-based} causal discovery, particularly \astar-based methods. Furthermore, while previous work has focussed on one, or two, types of domain knowledge, our approach allows for integration of all of the aforementioned types of domain knowledge.

% We here discuss how the formal definitions of the different types of domain knowledge allow us to reduce the search space in score-based causal discovery.
% Domain knowledge encodes expert knowledge about the underlying causal relationships in the data, which can help us solve causal discovery more efficiently. 
Previous exact-search score-based causal discovery methods that could not leverage domain knowledge have had to explore the entire graph search space. By effectively integrating domain knowledge, however, we are able to constrain this search space and, therefore, speed up causal discovery.
%Each type of domain knowledge described previously allows us to  we now turn towards its integration into score-based causal discovery. 
In this work we are concerned with {\astar}-based causal discovery, which mainly consists of two steps: 1) constructing the parent graphs, including computing all relevant scores, and 2) traversing these parent graphs using {\astar} to find shortest paths, which can then be combined to yield optimal parent sets. Below, we discuss how integration of the aforementioned types of domain knowledge affects each step.

\subsubsection{Pruning of the parent graphs} As explained previously, parent graphs are ordered graphs of all potential parent sets, as shown in Figure~\ref{fig:parent_graph_example}. Let us now assume that we are provided with some domain knowledge that is in direct violation with some of these potential parent sets. Since we are assuming that our domain knowledge is correct, we can prune away those potential parent sets that are causing the violation, only leaving parent sets that are allowed by the provided domain knowledge. 
% Repeating this for all variables leaves us with a set of parent graphs that only contain parent sets that are allowed by the domain knowledge. 
Consequently, we do not need to perform expensive score function evaluations for those potential parent sets that have been pruned away, which can result in significant computational gains.

In this work, we consider several different types of domain knowledge, specifically known edges, forbidden edges and tiers. First, consider the case of incorporating a known edge $x_i \rightarrow x_j$ into the causal discovery process. By definition, $x_i$ has to be included in the parent set of $x_j$, which means that the parent graph $\mathcal{G}_{Pa}(x_j)$ can be pruned in such a way that it only contains parent sets that incorporate $x_i$. Conversely, the optimal parent set of $x_i$ cannot include $x_j$ and, therefore, all potential parent sets in the parent graph $\mathcal{G}_{Pa}(x_i)$ that contain $x_j$ can be pruned away. We visualize this process in the left of Figure~\ref{fig:parent_graph_pruned} for a small parent graph.

Next, consider the case of incorporating a forbidden edge between two variables $x_i$ and $x_j$, implying that neither of them can be in the parent set of the other. Similar to before, we can prune away all potential parent sets in their parent graphs that contain the other variable, such that no parent sets in $\mathcal{G}_{Pa}(x_i)$ contain $x_j$, and vice versa. This process is visualized for a small parent graph in the right of Figure~\ref{fig:parent_graph_pruned}.

Lastly, let us consider domain knowledge that includes tiers, i.e.~a hierarchy of variables. This type of domain knowledge specifies that variables in tier $\mathbf{T}_k$ may be parents of variables in $\mathbf{T}_s$, for $k < s$, but members in $\mathbf{T}_s$ cannot be parents of members in $\mathbf{T}_k$. Therefore, assuming that a variable $x_i$ is in tier $\mathbf{T}_k$, we can prune its parent graph $\mathcal{G}_{Pa}(x_i)$ such that no potential parent sets contain any members in later tiers $\mathbf{T}_{>k}$. 
% Since this type of domain knowledge can be specified in many ways, with the extreme case being $p$ tiers for $p$ variables, 
We here focus on the common and practical scenario of having one source variable and one sink variable. This results in three tiers: 1) a source variable, 2) all other variables and 3) a sink variable. In this setting, the parent graph of the source variable can be pruned such that it only contains the empty set $\varnothing$. Similarly, the parent graphs of all variables besides the source and sink variables can be pruned such that they do not 
%contain any potential parent sets that 
include the sink variable. See the supplementary materials for more information on the general case of arbitrary tiers.

\subsubsection{Faster shortest-path finding} 

Using domain knowledge to prune parent graphs in {\astar}-based causal discovery, as explained above, directly reduces the number of necessary score function evaluations. In addition, pruning parent graphs further allows us to reduce the number of possible paths from the top node to the bottom node in the parent graph.
% (see Figure~\ref{fig:parent_graph_example}). 
This implies an important speed-up during the second step of {\astar}-based methods, which is concerned with finding the shortest-path from the top to bottom node using the {\astar} algorithm. Unfortunately, it is difficult to estimate the impact of reducing the number of potential paths, due to the heuristic function that {\astar} uses to limit the search space. Furthermore, {\astar} may ignore entire parts of the parent graph if it notices that the scores become worse as it gets towards the bottom node~\citep[see][]{yuan2013}. This may result in the bottom node not being the final node during the shortest path-finding, further complicating an analysis of how domain knowledge exactly impacts this procedure.  Nonetheless, this operation reduces the path-finding search space in the worst-case scenario and could therefore have a significant impact, especially for a large number of variables where the space of possible paths is exponentially large. As an example, consider the parent graph in Figure~\ref{fig:parent_graph_example} and let us insert a known edge, as done in Figure~\ref{fig:parent_graph_pruned}, which results in a reduction of the number of top-to-bottom paths from $6$ to $2$.
% Here, the number of paths from the top to bottom of the parent graph reduces from $6$ to $2$ when inserting domain knowledge. 
% While this effect may appear limited for such a small parent graph, it may be significant for large parent graphs where the number of paths grows exponentially.

\subsubsection{Applying Meek's rules}

{\astar}-based causal discovery algorithms are generally only guaranteed to recover the true completed partially directed acyclic graph (CPDAG)~\citep[see][]{ng2021}, which represents the Markov equivalence class (MEC) of the underlying true DAG. Often, this CPDAG is computed directly from a DAG estimate that is obtained after running {\astar}-based causal discovery~\citep[as in][]{ng2021}. While the original DAG estimate may be consistent with the provided domain knowledge, this may not be true for the corresponding CPDAG. As such, we here propose to apply the common rule set of~\citet{meek1995} to insert the provided domain knowledge into the CPDAG obtained from {\astar}-based causal discovery. Consequently, this \emph{modified CPDAG} no longer represents the MEC of the true DAG. 
% However, since we assume the domain knowledge to be correct, the modified CPDAG should still be a member of the same MEC as the true DAG.

% We here assume that there is no conflicting domain knowledge. For instance, we cannot specify that X is a direct causal driver of Y when we have already specified that X and Y should not be connected. This is important, because it ensures that the parent graph for a particular variable always has at least one node, since the empty set is also an element of the parent graph. This, in turn, implies that \astar-based causal discovery properly terminates if any domain knowledge is provided.

\subsection{Computational gains}

\begin{figure*}[t!]
    \centering
    \includegraphics[width=0.97\linewidth]{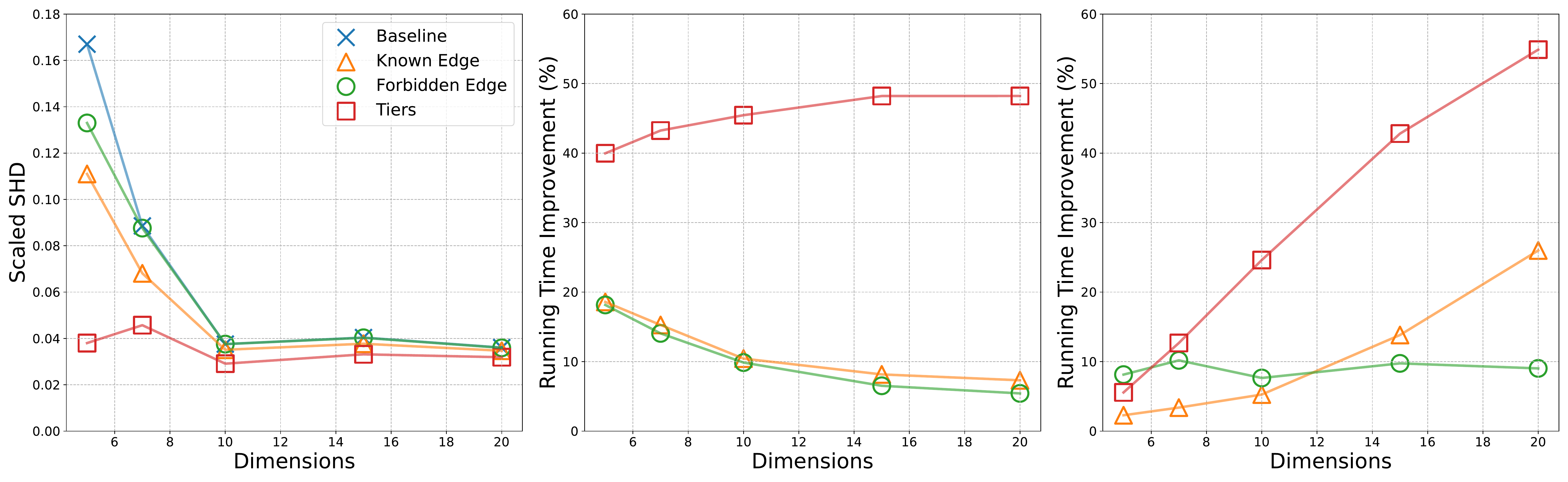}
    \caption{Results for applying {\astar}-based causal discovery on the synthetic data set when incorporating different types of domain knowledge. \emph{Left:} Performance improvements as measured by the scaled SHD (lower is better) for varying dimensions, averaged over $100$ repeats. \emph{Middle:} Percentage improvement in running times of regular {\astar} (middle) for varying dimensions, as compared to not incorporating any domain knowledge, averaged over $100$ repeats. 
    \emph{Right:} Percentage improvement in running times of {\astar}-SuperStructure for varying dimensions, as compared to not incorporating any domain knowledge, averaged over $100$ repeats. 
    % Note how the behaviour of {\astar} and {\astar}-SuperStructure can be drastically different from each other and for applying different types of domain knowledge.
    }
    \label{fig:synthetic_results}
\end{figure*}

We have previously explained how the number of score function evaluations can be reduced by pruning the parent graphs according to the domain knowledge. The exact number of final score function evaluations, however, is difficult to predict for any combination of domain knowledge. Nonetheless, for certain types of domain knowledge it is possible to provide estimates of the computational gains by deriving the exact reduction in score function evaluations. Additionally, the pruned parent graphs also result in a potentially faster executing of the \astar shortest-path finding algorithm, since fewer paths need to be searched. Quantifying the complexity of this operation, and the corresponding gains, is difficult due to the heuristic function in \astar~\citep{lu2021, ng2021}, which is why we focus on the reduction in score function evaluations to provide us with an estimate of the overall computational gains. In the interest of space, proofs of the following lemmas and theorems appear in the supplementary material.

We here only consider regular {\astar} and not the variations {\astar}-SuperStructure or Local {\astar}, since these rely on a super-structure that also results in a pruning of the parent graphs.
% in {\astar}-based causal discovery. 
Because estimation of this super-structure is data-driven, it is not possible to predict the extent of this pruning and its interaction with the pruning resulting from domain knowledge. We note, however, that {\astar} is a special worst-case scenario of {\astar}-SuperStructure and Local {\astar}~\citep{ng2021}.
% the computational gains for {\astar}-SuperStructure and Local {\astar} should be strictly less than for regular {\astar}, since the latter involves less parent graph pruning by design and can be seen as a special worst-case scenario of {\astar}-SuperStructure and Local {\astar}~\citep{ng2021}.

% Before considering the insertion of domain knowledge, 
First, we quantify the number of score function evaluations in the worst-case scenario of regular \astar-based causal discovery by means of the following lemma.

\begin{lemma}
A full parent graph $\mathcal{G}_{\text{PA}}$ has $2^{(p-1)}$ nodes, each representing a different potential parent set. Given a data set containing $p$ variables, \astar-based causal discovery requires $p2^{(p-1)}$ score function evaluations in the worst case.
\end{lemma}

Let us then consider the case of incorporating known, directed edges into {\astar}-based causal discovery. Theorem~\ref{theorem:known} provides an estimate of the computational gains when incorporating one known, directed edge.

\begin{theorem} \label{theorem:known}
Given a set of full parent graphs, incorporating the domain knowledge of one directed edge reduces the maximum number of score function computations in regular {\astar}-based causal discovery from $p2^{p-1}$ to $(p-1)2^{p-1}$, where $p$ is the total number of nodes.
\end{theorem}

% See the Appendix for a combinatorics-based proof. 
% As an example, for $10$ variables we would see the number of score-functions reduce from $5{,}120$ to $4{,}608$, i.e.~a reduction of $10\%$. 
% We emphasize again that this theorem, as well as the ones below, provide a, potentially biased, estimate of the overall computational gains, as represented by a reduction in score-function evaluations. The original {\astar} algorithm uses a heuristic function that can significantly speed up the exact-search, but whose interaction with the pruned parent graphs is difficult to analyze, as noted by~\citet{lu2021} and~\citet{ng2021}.

Next, we consider incorporating a forbidden edge into {\astar}-based causal discovery, with Theorem~\ref{theorem:forbidden} providing an estimate of the corresponding computational gains.

\begin{theorem} \label{theorem:forbidden}
Given a set of full parent graphs, incorporating the domain knowledge of one forbidden edge reduces the maximum number of score function computations in regular {\astar}-based causal discovery from $p2^{p-1}$ to $(p-1)2^{p-1}$, where $p$ is the total number of nodes.
\end{theorem}

% See the Appendix for a combinatorics-based proof of this theorem. 
Interestingly, this estimate of the computational gains for incorporating a forbidden edge is exactly the same as that for incorporating a known edge. This may provide useful advice for practitioners since, in the age of big data, causal graphs in practice tend to be sparse. A domain expert may, therefore, find it easier to specify a forbidden edge than a known edge. Since both cases yield the same computational gains, domain experts may therefore want to first focus on specifying as many forbidden relationships as possible.

Lastly, we consider the case of incorporating tiered domain knowledge into {\astar}-based causal discovery. As mentioned before, we here focus on the common and practical case of having one source node and one sink node, ultimately resulting in three tiers. Theorem~\ref{theorem:tiers} provides an estimate of the ensued computational gains (see the supplementary material for a theorem with tiers of arbitrary sizes).

\begin{theorem} \label{theorem:tiers}
Given a set of full parent graphs, incorporating tiered domain knowledge that includes exactly one source variable and one sink variable, resulting in three tiers, halves the maximum number of score function computations in regular {\astar}-based causal discovery.
\end{theorem}

% See the Appendix for a combinatorics-based proof. 
The above theorem implies that the potential computational gains from incorporating tiered domain knowledge far outweigh those from incorporating known or forbidden edges, especially when the underlying causal graph is large. Theorem~\ref{theorem:tiers} therefore suggests that specifying tiered domain knowledge should be a priority for practitioners.
%, even if it may be more challenging to specify. 
% From a practical point of view, however, it is generally also harder to identify sink and source nodes with certainty.

\section{Experiments}

In this section we provide experimental results of the computational gains obtained when integrating several types of domain knowledge with {\astar}-based causal discovery methods. We first consider synthetic data that has been generated using randomly-sampled DAGs.
% of varying dimensions. 
We then study the application of our approach on the commonly-used, real Sachs data set~\citep{sachs2005}.
%, which is a real data set with an associated expert-formed ground-truth DAG.
See the supplementary material for more detailed explanations of the experimental setup.

\subsection{Evaluation}

We evaluate the performance of all causal discovery methods by their ability to recover the true graph, as measured by the structural hamming distance (SHD) between \emph{modified} CPDAGs. 
% Recall that {\astar}-based methods return a \emph{modified} CPDAG, where Meek's rules have been applied post-discovery to ensure that the resulting CPDAG complies with the provided domain knowledge. To ensure a fair comparison, we therefore need to compute the CPDAG of the ground-truth DAG and then similarly apply Meek's rules. This allows us to measure causal discovery performance via the SHD between the modified CPDAG estimate and the modified CPDAG of the ground-truth.
Where applicable, we also scale the SHD by the number of possible edges $p (p - 1) / 2$.
% in order to appropriately compare efficacy across different dimensions.
Lastly, we also assess all methods by their computational running time.
% on a commercial CPU.

\subsection{Synthetic Data}

% In order to thoroughly assess the impact of inserting different types of domain knowledge, 
We here consider synthetic data that has been generated using a linear structural equation model (SEM) with Gaussian noise and a randomly-sampled Erd\H{o}s-R\'{e}nyi graph. For all synthetic data experiments, we sample ground-truth DAGs with an average degree of $2$ and varying dimensions of $p \in \{ 5, 7, 10, 15, 20\}$. For each DAG, we use a linear SEM with randomly-sampled parameters to generate $500$ observations. Throughout our experiments, known edges are randomly sampled from the set of true edges, whereas forbidden edges are randomly sampled from its complement set. Similarly, the source and sink node for the tiers domain knowledge are randomly sampled from the set of true source and sink nodes, respectively. 
% See the supplementary material for more information on the experimental setup.

%and 2) computational running time. Since  

\subsubsection{Results}

Figure~\ref{fig:synthetic_results} summarizes the performance and computational gains obtained when integrating different types of domain knowledge with {\astar}-based causal discovery. The left figure shows the scaled SHD as a function of data dimensions for regular {\astar}-based causal discovery.
% \footnote{See the supplementary for plots of the unscaled SHD.} 
While all approaches result in smaller scaled SHD for increasing dimensions, it is apparent that integrating domain knowledge can largely reduce the SHD across all dimensions. The effect of incorporating tiers appears to be the highest, followed by that of incorporating known edges and then, lastly, that of incorporating forbidden edges. Using Student's t-tests, we are able to ascertain that the SHD improvements for tiers and known edges are, in fact, statistically significant with a p-value threshold of $0.05$, while those for forbidden edges are not (see the supplementary material for more information). 
% This result can be explained by considering that the SHD measures the distance from our estimated graph to the ground-truth graph. While tiers contains inherently more information about the underlying causal relationships, known edges specify true edges directly which helps in reducing the SHD. Since the correctly-identified absence of an edge is not explicitly rewarded, specifying forbidden edges, on the other hand, do not necessarily help in reducing the SHD.

Next, we turn towards percentage improvements in the running time of {\astar}-based causal discovery, as shown in the middle and right of Figure~\ref{fig:synthetic_results}. The middle plot shows how incorporating different types of domain knowledge can help speed up regular {\astar}. The computational gains for incorporating known edges and for incorporating forbidden edges fall off inversely proportional to the data dimensionality, which is entirely in line with our theoretical analysis in the methodology section, i.e.~Theorems~\ref{theorem:known} and~\ref{theorem:forbidden}. For tiers, the percentage improvement starts off at around 40\% and then increases to 50\% at $15$ and $20$ dimensions, which matches the gain predicted by Theorem~\ref{theorem:tiers}. While there is a slight mismatch for small dimensions, the corresponding confidence intervals (see the supplementary material) confidently capture the expected gain.

Similarly, the right plot in Figure~\ref{fig:synthetic_results} shows how incorporating domain knowledge can help speed up {\astar}-SuperStructure (we show a similar plot for Local {\astar} in the supplementary material). Interestingly, the behavior for {\astar}-SuperStructure appears largely different to that of regular {\astar}. Here, the computational gains for incorporating known edges and incorporating tiers grow nearly linearly with the dimensions, whereas the running time improvements for incorporating forbidden edges stay roughly constant. 
% We suspect this is due to graphical LASSO pruning the parent graph in the same way that integrating forbidden edges. 
Recall that {\astar}-SuperStructure has a data-driven skeleton-estimation step which can have large computational overhead. In fact, this overhead may be the computational bottleneck of {\astar}-SuperStructure at small dimensions, as can been seen when looking at absolute running times as a function of dimensions (see the supplementary material). While integrating domain knowledge helps in reducing the number of score function evaluations, it does not help with the data-driven estimation of the super structure in {\astar}-SuperStructure. Moreover, as noted previously, Theorems~\ref{theorem:known}$-$\ref{theorem:tiers} are only valid for regular {\astar}-based causal discovery, because of the unclear interaction of the super-structure induced parent graph pruning and the domain knowledge induced pruning.

% We first consider synthetic data that has been generated using linear structural equation models (SEMs) of randomly-sampled Erd\H{o}s-R\'{e}nyi graphs.

%\begin{figure}[t!]
%    \centering
%    \includegraphics[width=0.95\linewidth]{images/scaled_shd_astar_comparison.pdf}
%    \caption{Performance improvements when applying different types of domain knowledge into %regular {\astar}-based causal discovery on a synthetic data set. Shown is the scaled SHD %%(lower is better) with varying number of dimensions, averaged over $100$ repeats. 
%    %The SHD was computed between the modified CPDAG estimate and the modified CPDAG of the %ground-truth DAG, and then scaled by the number of possible edges for each dimension.
%    }
%    \label{fig:shd_synthetic}
%\end{figure}

\subsection{Sachs Data}

\begin{figure}
\centering
    \includegraphics[width=1\linewidth]{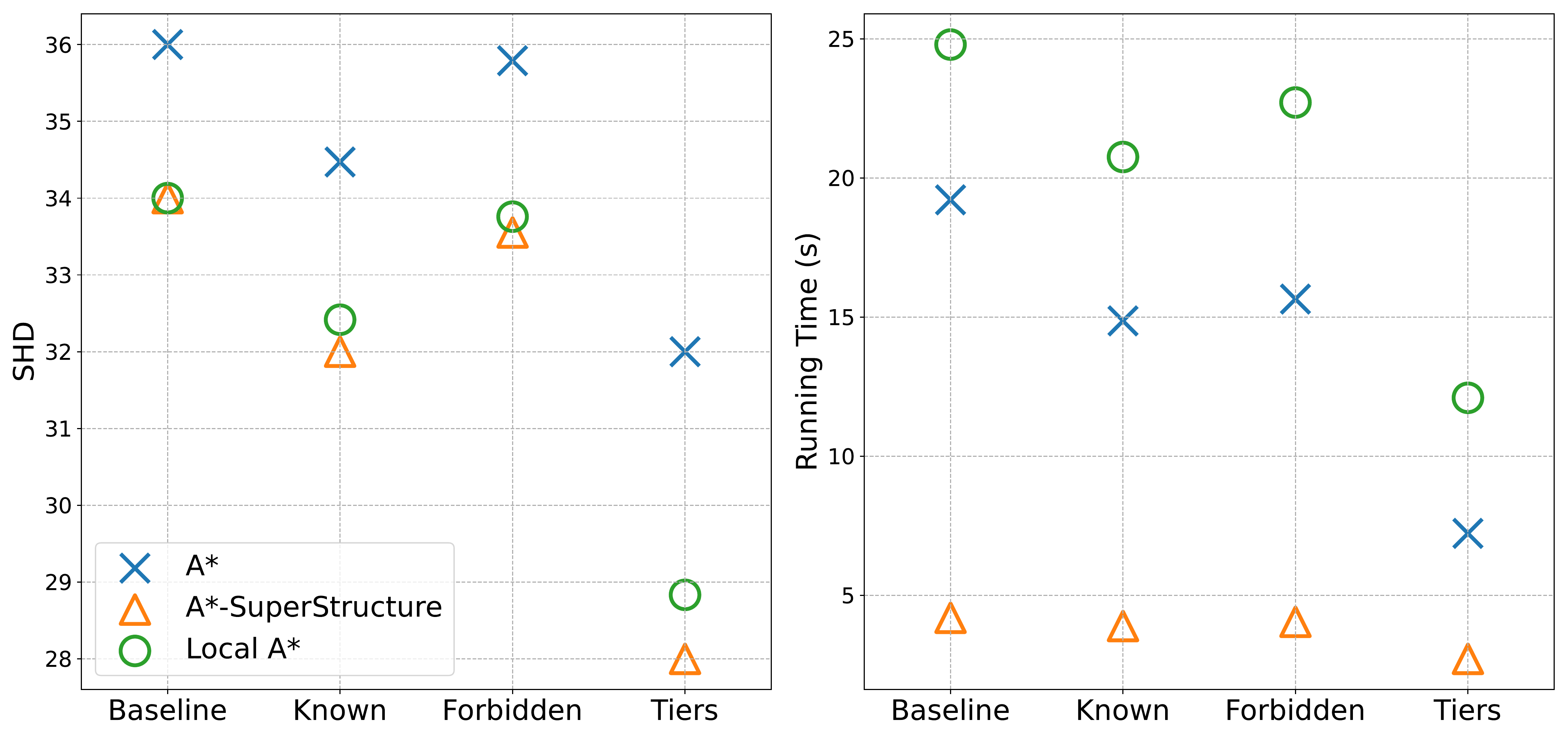}
    \caption{Results for applying {\astar}-based causal discovery on the Sachs data set for different types of domain knowledge. 
    % \caption{Results for applying {\astar}-based causal discovery on the real protein signalling data set of~\citet{sachs2005} for different types of domain knowledge. 
    % Shown are the average SHD  (left) and the average running time (right). 
    }
    \label{fig:sachs_results}
\end{figure}

We here consider the real protein signalling data set of~\citet{sachs2005}, which describes protein interactions in human cells. The data set contains $11$ continuous variables with $7{,}466$ observations and, importantly, has a known ground-truth DAG that was formed by experts with the help of interventional data (see the supplementary material).
%(see Figure~\ref{fig:sachs_graph}).

Similar to before, we are concerned with analyzing the impact of integrating different types of domain knowledge with {\astar}-based causal discovery, specifically that of known edges, forbidden edges and tiers. To this end, we first generate all possible combinations of such domain knowledge, as given by the known ground-truth. For incorporating known edges, this implies running {\astar}-based causal discovery for every possible edge that could be inserted ($17$ in total). Similarly, for incorporating forbidden edges, this means running causal discovery for all non-edges (of which there are $38$). There are $2$ source variables and $4$ sink variables in the ground-truth DAG, resulting in $8$ different sets of tiers.

In Figure~\ref{fig:sachs_results} we summarize the results of integrating these sets of domain knowledge with regular {\astar}, {\astar}-SuperStructure and Local {\astar}. The left plot shows how the SHD is affected by considering different types of domain knowledge and different {\astar}-based methods. All types of domain knowledge generally improve SHD, with tiers having the largest impact, followed by known edges and then forbidden edges, which is similar to results of the synthetic data set experiments. Notably, {\astar}-SuperStructure appears to be affected most by the insertion of domain knowledge and also results in the largest SHD improvements. Even though both {\astar}-SuperStructure and Local {\astar} rely on graphical LASSO to compute a super-structure, the latter does not seem to be affected as much as the former. 
% We also note that the SHD values shown in  the left of Figure~\ref{fig:sachs_results} appear to be slightly higher than those reported in other works~\citep[e.g.][]{ke2022}, who reported the SHD between the ground-truth DAG and an estimated DAG. In this work, however, we computed the SHD between the modified CPDAG estimate and the modified CPDAG of the ground-truth DAG.

Finally, the right plot in Figure~\ref{fig:sachs_results} shows absolute running times for {\astar}-based causal discovery when incorporating different types of domain knowledge. Regular {\astar} seems to be affected the most by the insertion of domain knowledge, whereas {\astar}-SuperStructure does not seem to be affected much at all. One of the reasons for this observation may be the potentially large overhead due to the super-structure computation, which was already observed for relatively small synthetic data sets (see Figure~\ref{fig:sachs_results}). While Local {\astar} appears to be the slowest method, we note that this method
% all of our experiments were run on a single CPU and, unlike the other methods, Local {\astar} 
has the potential to be parallelized over several CPUs~\citep{ng2021}.
% , which could further improve its running time. 
Overall, every method appears to be positively affected by domain knowledge, with respect to both SHD and running time.

\section{Conclusions}

In this work we provided a sound method for integrating different types of domain knowledge with {\astar}-based causal discovery methods. In doing so, we discussed how domain knowledge can affect each step of the causal discovery process and offered a theoretical analysis of the potential computational gains, as measured by a reduction in the number of score function computations. We supported our findings with synthetic and real data, finding that domain knowledge can help improve recovery of the ground-truth and computational running time in nearly every setting. Furthermore, we showed that various extensions, such as {\astar}-SuperStructure and Local {\astar}, can show significantly different responses to the integration of domain knowledge than regular {\astar}.

One of the current main limitations of {\astar}-based causal discovery is that it is only guaranteed to converge to the global optimum in the linear and Gaussian setting. In the future, it may be interesting to investigate, theoretically and empirically, how the integration of domain knowledge affects the non-linear and non-Gaussian setting, perhaps even considering discrete data.

% \section{Ethical Statement}
% \section{Acknowledgements}
\bibliography{aaai}

\clearpage

\appendix
\newcommand{\appendixhead}%
{\centering\textbf{\LARGE Technical Appendix for Domain Knowledge in A*-Based Causal Discovery}
\vspace{0.25in}}
\twocolumn[\appendixhead]

\section{Additional theoretical contributions}

\subsection{Score computations in a full parent graph}

\begin{lemma} \label{lemma:app_pa}
A full parent graph $\mathcal{G}_{\text{PA}}$ has $2^{(p-1)}$ nodes, each representing a different potential parent set. Given a data set containing $p$ variables, \astar-based causal discovery requires $p2^{(p-1)}$ score function evaluations in the worst case.
\end{lemma}

\begin{proof}
Consider a set of $p$ variables $X = \{x_1, \dots, x_p\}$ and the full parent graph $\mathcal{G}_{Pa}(x_i)$ of variable $x_i$, which is an ordered graph of all potential parent sets of $x_i$ in a causal graph. Let us define a level $l$ within the parent graph where each potential parent set has exactly $l$ elements. By definition, the full parent graph $\mathcal{G}_{Pa}(x_i)$ contains potential parent sets with an increasing number of elements, starting from the empty set $\varnothing$ (top level) to all other variables $X \backslash x_i$ (bottom level), which means that $\mathcal{G}_{Pa}(x_i)$ contains $p$ levels $l \in \{0, \dots, p-1\}$. In order to find out how many potential parent sets correspond to each level $l$, we need to count how many potential parent sets can been chosen from $X \backslash x_i$ such that they contain $l$ elements. This number can be obtained exactly by applying the choose operator $C(p-1, l)$. We can then compute how many nodes a full parent graph $\mathcal{G}_{Pa}(x_i)$ contains by summing over all levels $l$, i.e.
\begin{equation}
    |\mathcal{G}_{\text{PA}}(x_i)| = \sum_{l=0}^{p-1} C(p-1, l) = 2^{p-1},
\end{equation}
where the second step is a direct consequence of the well-known binomial theorem. Since we have $p$ of such full parent graphs, the number of score function evaluations in {\astar}-based causal discovery is precisely $p2^{p-1}$.
\end{proof}

\subsection{Computational gains for known edges}

\begin{theorem}
Given a set of full parent graphs, incorporating the domain knowledge of one directed edge reduces the maximum number of score function computations in regular {\astar}-based causal discovery from $p2^{p-1}$ to $(p-1)2^{p-1}$, where $p$ is the total number of nodes.
\end{theorem}

\begin{proof}
Let us insert a known edge from variable $x_i$ to variable $x_j$, where $i, j \in \{1, \dots, p\}$ and $i\neq j$. Since $x_i$ is assured to be a parent of $x_j$, we do not need to compute scores for nodes, i.e.~potential parent sets, in the parent graph of $x_i$ and $x_j$ that violate this known causal relationship.

First, consider the full parent graph $\mathcal{G}_{Pa}(x_i)$ of $x_i$. By incorporating the domain knowledge that $x_i \rightarrow x_j$, we know that $x_j$ cannot be a parent of $x_i$ and, therefore, we can remove all nodes in the parent graph of $x_i$ that contain $x_j$. This changes the number of potential parents of $x_i$ from $p-1$ to $p-2$, which, according to Lemma~\ref{lemma:app_pa}, reduces the number of score computations in this parent graph from $2^{p-1}$ to $2^{p-2}$.

Second, consider the parent graph $\mathcal{G}_{Pa}(x_j)$ of $x_j$ in a similar manner. By incorporating the domain knowledge that $x_i \rightarrow x_j$, we know that $x_i$ has to be in the optimal parent set of $x_j$. This means that we can prune all nodes in this parent graph that do not contain $x_i$. In order to deduce the reduction in score computations, recall from Lemma~\ref{lemma:app_pa} that the number of nodes in the full parent graph $\mathcal{G}_{Pa}(x_j)$, without incorporating domain knowledge, is given by
\begin{equation}
|\mathcal{G}_{\text{PA}}(x_j)| = \sum_{l=0}^{p-1} C(p-1, l) = 2^{p-1},
\end{equation}
where $C(p-1, l)$ denotes the number of combinations that we can choose $l$ parents out of $p-1$ potential parents. Incorporating the known edge $x_i \rightarrow x_j$ implies that the parent set of $x_j$ cannot be empty, which means that the $l=0$ term can be dropped from the above sum. Furthermore, since $x_i$ is a required parent it needs to be included in every potential parent set, which means that we only need to choose $l-1$ other parents from a total pool of $p-2$ parents for every term in the above sum. The variable $x_i$ can then simply be added to every chosen parent set to construct a valid parent set. The number of score computations for this pruned parent graph $\mathcal{G}_{Pa}(x_j)$ is therefore given by
\begin{align}
|\mathcal{G}_{\text{PA}}(x_j)| 
&=\sum_{l=1}^{p-1} C(p-2, l-1) \\
&=\sum_{l^{\prime}=1}^{p-2} C(p-2, l^{\prime}) \\
&= 2^{p-2},
\end{align}
where made a change of variables $l^{\prime} = l-1$ in the second line and applied Lemma~\ref{lemma:app_pa} in the third line.

Therefore, the total number of score function evaluations across all pruned parent graphs is
\begin{equation}
    2^{p-2} + (p-2) 2^{p-1} + 2^{p-2} = (p-1) 2^{p-1},
\end{equation}
where the first and third term are due to the parent graphs of $x_i$ and $x_j$, respectively, and the second term is due to the parent graphs of all the other nodes $X \backslash \{x_i, x_j\}$, which remain unchanged.
\end{proof}

\subsection{Computational gains for forbidden edges}

\begin{theorem}
Given a set of full parent graphs, incorporating the domain knowledge of one forbidden edge reduces the maximum number of score function computations in regular {\astar}-based causal discovery from $p2^{p-1}$ to $(p-1)2^{p-1}$, where $p$ is the total number of nodes.
\end{theorem}

\begin{proof}
Let us insert a forbidden edge between variable $x_i$ and variable $x_j$, where $i, j \in \{1, \dots, p\}$ and $i\neq j$. Since neither variable can be in the parent set of the other variable, we do not need to compute scores for nodes, i.e.~potential parent sets, in the parent graph of $x_i$ and $x_j$ that violate this known causal relationship.

First, consider the full parent graph $\mathcal{G}_{Pa}(x_i)$ of $x_i$. Since we know that $x_j$ cannot be a parent of $x_i$, we can remove all nodes in the parent graph of $x_i$ that contain $x_j$. This changes the number of potential parents of $x_i$ from $p-1$ to $p-2$, which, according to Lemma~\ref{lemma:app_pa}, reduces the number of score computations in this parent graph from $2^{p-1}$ to $2^{p-2}$. The exact same argument then applies to the parent graph $\mathcal{G}_{Pa}(x_j)$ of $x_j$, which can be pruned such that it does not contain any potential parent sets that involve $x_i$. This similarly reduces the number of score computations from $2^{p-1}$ to $2^{p-2}$.

Therefore, the total number of score function evaluations across all pruned parent graphs is
\begin{equation}
    2^{p-2} + (p-2) 2^{p-1} + 2^{p-2} = (p-1) 2^{p-1},
\end{equation}
where the first and third term are due to the parent graphs of $x_i$ and $x_j$, respectively, and the second term is due to the parent graphs of all the other nodes $X \backslash \{x_i, x_j\}$, which remain unchanged.
\end{proof}

\subsection{Computational gains for tiers with one source and one sink variable}

\begin{theorem}
Given a set of full parent graphs, incorporating tiered domain knowledge that includes exactly one source variable and one sink variable, resulting in three tiers, halves the maximum number of score function computations in regular {\astar}-based causal discovery.
\end{theorem}

\begin{proof}
Consider the tiered domain knowledge $\mathbf{T} = \{\mathbf{T}_1, \mathbf{T}_2, \mathbf{T}_3\}$, where $\mathbf{T}_1 = \{x_i\}$ contains a single source variable $x_i$, $\mathbf{T}_3 = \{x_j\}$ contains a single sink variable $x_j$ and $\mathbf{T}_2 = \{x_1, \dots, x_p\} \backslash \{x_i, x_j\}$ consists of all other variable except the source and sink variables. Here, the tiers $\mathbf{T}$ specifies that directed edges can never go from elements in $\mathbf{T}_3$ to elements in either $\mathbf{T}_1$ or $\mathbf{T}_2$, or from elements in $\mathbf{T}_2$ to elements in $\mathbf{T}_1$.
%Here, the tiers $\mathbf{T}$ specifies that directed edges can only go from elements in $\mathbf{T}_1$ to elements of $\mathbf{T}_2$ or $\mathbf{T}_3$, or from elements in $\mathbf{T}_2$ to elements in $\mathbf{T}_3$, but not the other way around. 
We can deduce the reduction in the maximum number of score function computations by pruning the set of parent graphs such that they conform with $\mathbf{T}$.

Consider the full parent graph $\mathcal{G}_{Pa}(x_i)$ of the source variable $x_i$. Since $x_i$ is contained in the first tier $\mathbf{T}_1$, the domain knowledge specifies that it has no parents and therefore we do not need to perform any score computations in its parent graph at all. In contrast, the full parent graph $\mathcal{G}_{Pa}(x_j)$ of the sink variable $x_j$ remains unchanged, since it is contained in the last tier $\mathbf{T}_3$, which means that its maximum number of score computations is still $2^{p-1}$.

Let us now consider the full parent graph $\mathcal{G}_{Pa}(x_k)$ of any variable $x_k \in \mathbf{T}_2$. Since there cannot be directed edges from the sink variable $x_j \in \mathbf{T}_3$ to $x_k$, we need to prune all nodes in the full parent graph that contain the variable $x_j$. This reduces the number of potential parents from $p-1$ to $p-2$, which, according to Lemma~\ref{lemma:app_pa}, means that the parent graph of $x_k$ has $2^{p-2}$ nodes. Since there are $p-2$ variables in the second tier $\mathbf{T}_2$, this results in a maximum of $(p-2)2^{p-2}$ score computations.

Adding the number of score computations for all full parent graphs yields $(p-2)2^{p-2} + 2^{p-1} = \frac{p}{2}2^{p-1}$ number of score computations, which corresponds to an exact reduction by half.
\end{proof}

\subsection{Computational gains for arbitrary tiers}

In this section we provide an extension of the above theorem to an arbitrary number of tiers with arbitrary sizes.

\begin{theorem}
Given a set of full parent graphs, incorporating tiered domain knowledge that includes $n$ tiers $\{\mathbf{T}_1, \dots, \mathbf{T}_n\}$ reduces the maximum number of score function computations in regular {\astar}-based causal discovery from $p2^{p-1}$ to the following:
\begin{equation}
    \sum_{k=1}^{n-1} |\mathbf{T}_k| 2^{p - 1 - \sum_{l=k+1}^n |\mathbf{T}_l|} + |\mathbf{T}_n| 2^{p - 1},
\end{equation}
where $|\mathbf{T}_k|$ is the number of variables in tier $k$.
\end{theorem}

\begin{proof}
Consider the tier $\mathbf{T}_k$ containing $|\mathbf{T}_k|$ variables, including the variable $x_i \in \mathbf{T}_k$. By the definition of tiered domain knowledge, any variable $x_j$ that is a member of a later tier $\mathbf{T}_{l}$, i.e.~$k<l$ cannot be a parent of $x_i$. The number of such variables $x_j$ that cannot be parents of $x_i$ is obtained by summing over the sizes of all subsequent tiers, i.e.~$\sum_{l=k+1}^n |\mathbf{T}_l|$. The size of the parent graph for variable $x_i$ is then $2^{p-1-\sum_{l=k+1}^n |\mathbf{T}_l|}$, where $p$ are the total number of variables. Consequently, the total number of score function evaluations for tier $\mathbf{T}_k$ is $|\mathbf{T}_k| 2^{p - 1 - \sum_{l=k+1}^n |\mathbf{T}_l|}$. This holds true for all $1 \leq k < n$, but not for the last tier $\mathbf{T}_n$, where no later tiers exist and the total number of score function evaluations remains $|\mathbf{T}_n| 2^{p-1}$. Note that elements in the first tier $\mathbf{T}_1$ no longer have empty parent graphs since edges between elements in the same tier are allowed. The total number of score function evaluations is then:
\begin{equation}
    \sum_{k=1}^{n-1} |\mathbf{T}_k| 2^{p - 1 - \sum_{l=k+1}^n |\mathbf{T}_l|} + |\mathbf{T}_n| 2^{p - 1}.
\end{equation}
\end{proof}

\subsection{Visualizing the parent graph pruning}

In Figure 2 of the main text, we visualize parent graphs that have been pruned due to the insertion of domain knowledge, specifically that of known edges and forbidden edges. We here provide further visualizations of pruned parent graphs for all types of domain knowledge. Specifically, Figure~\ref{fig:app_parent_graph_known} shows the pruned parent graphs of $x_1$ (left) and $x_4$ (right) when incorporating the known edge $x_4 \rightarrow x_1$, while Figure~\ref{fig:app_parent_graph_forbidden} shows the pruned parent graphs of $x_1$ (left) and $x_4$ (right) when incorporating a forbidden edge between them. Figure~\ref{fig:app_parent_graph_tiers} shows the parent graphs of $x_4$ (left) and $x_2$ (right) when incorporating tiers such that $x_1$ is a source variable and $x_4$ is a sink variable.

\begin{figure*}
\centering
\begin{minipage}{.47\linewidth}
\centering
\begin{tikzpicture}[
    every node/.append style={draw, ellipse, minimum width = 2cm, minimum height = 1cm, semithick}]
    \node[dotted] (empty) at (3,0) {\Large $\varnothing$};
    \node[dotted] (x2) at (0,-1.5) {$\{x_2\}$};
    \node[dotted] (x3) at (3,-1.5) {$\{x_3\}$};
    \node[fill=gray!10] (x4) at (6, -1.5) {$\{x_4\}$};
    \node[dotted] (x2x3) at (0,-3) {$\{x_2, x_3\}$};
    \node[fill=gray!10] (x2x4) at (3,-3) {$\{x_2, x_4\}$};
    \node[fill=gray!10] (x3x4) at (6,-3) {$\{x_3, x_4\}$};
    \node[fill=gray!10] (x2x3x4) at (3,-4.5) {$\{x_2, x_3, x_4\}$};

    \path[-latex, semithick]
    	(empty) edge[dotted, ->] (x2)
    	(empty) edge[dotted, ->] (x3)
    	(empty) edge[dotted, ->] (x4)
    	(x2) edge[dotted, ->] (x2x3)
    	(x2) edge[dotted, ->] (x2x4)
    	(x3) edge[dotted, ->] (x2x3)
    	(x3) edge[dotted, ->] (x3x4)
    	(x4) edge (x2x4)
    	(x4) edge (x3x4)
    	(x2x3) edge[dotted, ->] (x2x3x4)
    	(x2x4) edge (x2x3x4)
    	(x3x4) edge (x2x3x4);
\end{tikzpicture}
\end{minipage} \hfill
\begin{minipage}{0.47\linewidth}
\centering
\begin{tikzpicture}[
    every node/.append style={draw, ellipse, minimum width = 2cm, minimum height = 1cm, semithick}]
    \node[fill=gray!10] (empty) at (3,0) {\Large $\varnothing$};
    \node[dotted] (x1) at (0,-1.5) {$\{x_1\}$};
    \node[fill=gray!10] (x2) at (3,-1.5) {$\{x_2\}$};
    \node[fill=gray!10] (x3) at (6, -1.5) {$\{x_3\}$};
    \node[dotted] (x1x2) at (0,-3) {$\{x_1, x_2\}$};
    \node[dotted] (x1x3) at (3,-3) {$\{x_1, x_3\}$};
    \node[fill=gray!10] (x2x3) at (6,-3) {$\{x_2, x_3\}$};
    \node[dotted] (x1x2x3) at (3,-4.5) {$\{x_1, x_2, x_3\}$};

    \path[-latex, semithick]
    	(empty) edge[dotted, ->] (x1)
    	(empty) edge (x2)
    	(empty) edge (x3)
    	(x1) edge[dotted, ->] (x1x2)
    	(x1) edge[dotted, ->] (x1x3)
    	(x2) edge[dotted, ->] (x1x2)
    	(x2) edge (x2x3)
    	(x3) edge[dotted, ->] (x1x3)
    	(x3) edge (x2x3)
    	(x1x2) edge[dotted, ->] (x1x2x3)
    	(x1x3) edge[dotted, ->] (x1x2x3)
    	(x2x3) edge[dotted, ->] (x1x2x3);
\end{tikzpicture}
\end{minipage}
    \caption{Pruned parent graphs $\mathcal{G}_{Pa}(x_1)$ (left) and $\mathcal{G}_{Pa}(x_4)$ (right) when incorporating a known edge $x_4 \rightarrow x_1$.}
    \label{fig:app_parent_graph_known}
\end{figure*}
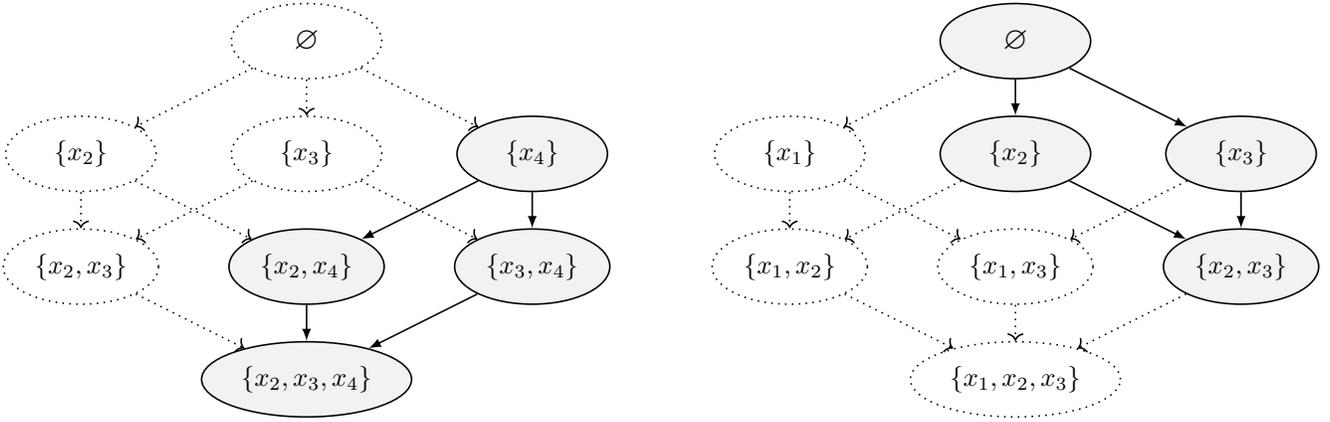

\begin{figure*}
\centering
\begin{minipage}{.47\linewidth}
\centering
\begin{tikzpicture}[
    every node/.append style={draw, ellipse, minimum width = 2cm, minimum height = 1cm, semithick}]
    \node[fill=gray!10] (empty) at (3,0) {\Large $\varnothing$};
    \node[fill=gray!10] (x2) at (0,-1.5) {$\{x_2\}$};
    \node[fill=gray!10] (x3) at (3,-1.5) {$\{x_3\}$};
    \node[dotted] (x4) at (6, -1.5) {$\{x_4\}$};
    \node[fill=gray!10] (x2x3) at (0,-3) {$\{x_2, x_3\}$};
    \node[dotted] (x2x4) at (3,-3) {$\{x_2, x_4\}$};
    \node[dotted] (x3x4) at (6,-3) {$\{x_3, x_4\}$};
    \node[dotted] (x2x3x4) at (3,-4.5) {$\{x_2, x_3, x_4\}$};

    \path[-latex, semithick]
    	(empty) edge (x2)
    	(empty) edge (x3)
    	(empty) edge[dotted, ->] (x4)
    	(x2) edge (x2x3)
    	(x2) edge[dotted, ->] (x2x4)
    	(x3) edge (x2x3)
    	(x3) edge[dotted, ->] (x3x4)
    	(x4) edge[dotted, ->] (x2x4)
    	(x4) edge[dotted, ->] (x3x4)
    	(x2x3) edge[dotted, ->] (x2x3x4)
    	(x2x4) edge[dotted, ->] (x2x3x4)
    	(x3x4) edge[dotted, ->] (x2x3x4);
\end{tikzpicture}
\end{minipage} \hfill
\begin{minipage}{0.47\linewidth}
\centering
\begin{tikzpicture}[
    every node/.append style={draw, ellipse, minimum width = 2cm, minimum height = 1cm, semithick}]
    \node[fill=gray!10] (empty) at (3,0) {\Large $\varnothing$};
    \node[dotted] (x1) at (0,-1.5) {$\{x_1\}$};
    \node[fill=gray!10] (x2) at (3,-1.5) {$\{x_2\}$};
    \node[fill=gray!10] (x3) at (6, -1.5) {$\{x_3\}$};
    \node[dotted] (x1x2) at (0,-3) {$\{x_1, x_2\}$};
    \node[dotted] (x1x3) at (3,-3) {$\{x_1, x_3\}$};
    \node[fill=gray!10] (x2x3) at (6,-3) {$\{x_2, x_3\}$};
    \node[dotted] (x1x2x3) at (3,-4.5) {$\{x_1, x_2, x_3\}$};

    \path[-latex, semithick]
    	(empty) edge[dotted, ->] (x1)
    	(empty) edge (x2)
    	(empty) edge (x3)
    	(x1) edge[dotted, ->] (x1x2)
    	(x1) edge[dotted, ->] (x1x3)
    	(x2) edge[dotted, ->] (x1x2)
    	(x2) edge (x2x3)
    	(x3) edge[dotted, ->] (x1x3)
    	(x3) edge (x2x3)
    	(x1x2) edge[dotted, ->] (x1x2x3)
    	(x1x3) edge[dotted, ->] (x1x2x3)
    	(x2x3) edge[dotted, ->] (x1x2x3);
\end{tikzpicture}
\end{minipage}
    \caption{Pruned parent graphs $\mathcal{G}_{Pa}(x_1)$ (left) and $\mathcal{G}_{Pa}(x_4)$ (right) when incorporating a forbidden edge between $x_1$ and $x_4$.}
    \label{fig:app_parent_graph_forbidden}
\end{figure*}
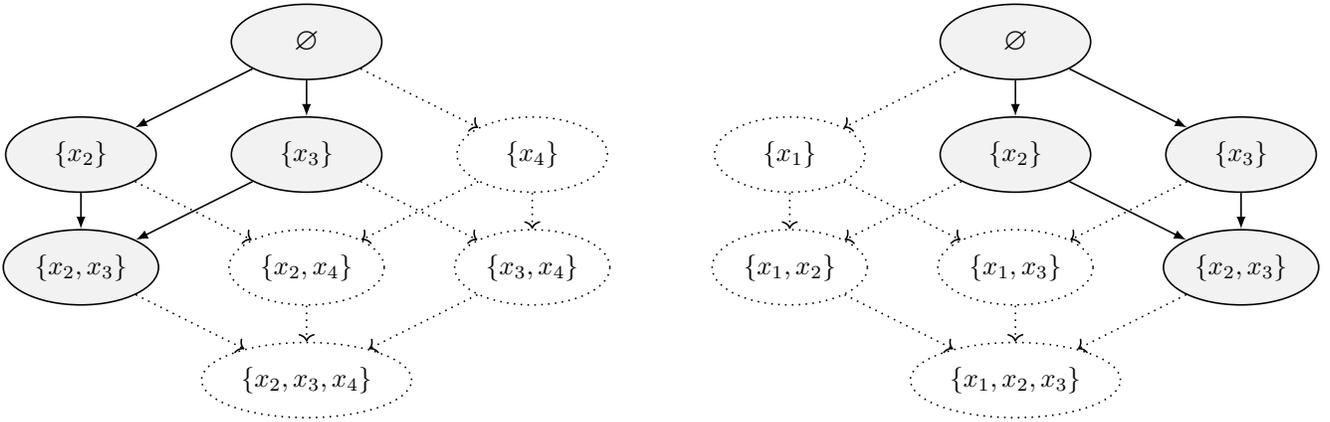

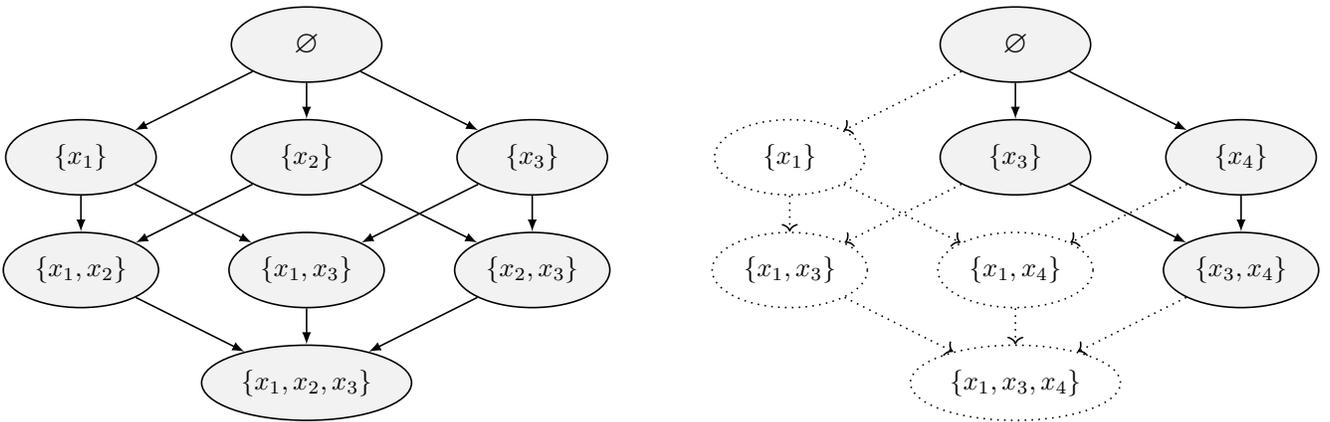
\begin{figure*}
\centering
\begin{minipage}{.47\linewidth}
\centering
\begin{tikzpicture}[
    every node/.append style={draw, ellipse, minimum width = 2cm, minimum height = 1cm, semithick}]
    \node[fill=gray!10] (empty) at (3,0) {\Large $\varnothing$};
    \node[fill=gray!10] (x1) at (0,-1.5) {$\{x_1\}$};
    \node[fill=gray!10] (x2) at (3,-1.5) {$\{x_2\}$};
    \node[fill=gray!10] (x3) at (6, -1.5) {$\{x_3\}$};
    \node[fill=gray!10] (x1x2) at (0,-3) {$\{x_1, x_2\}$};
    \node[fill=gray!10] (x1x3) at (3,-3) {$\{x_1, x_3\}$};
    \node[fill=gray!10] (x2x3) at (6,-3) {$\{x_2, x_3\}$};
    \node[fill=gray!10] (x1x2x3) at (3,-4.5) {$\{x_1, x_2, x_3\}$};

    \path[-latex, semithick]
    	(empty) edge (x1)
    	(empty) edge (x2)
    	(empty) edge (x3)
    	(x1) edge (x1x2)
    	(x1) edge (x1x3)
    	(x2) edge (x1x2)
    	(x2) edge (x2x3)
    	(x3) edge (x1x3)
    	(x3) edge (x2x3)
    	(x1x2) edge (x1x2x3)
    	(x1x3) edge (x1x2x3)
    	(x2x3) edge (x1x2x3);
\end{tikzpicture}
\end{minipage} \hfill
\begin{minipage}{0.47\linewidth}
\centering
\begin{tikzpicture}[
    every node/.append style={draw, ellipse, minimum width = 2cm, minimum height = 1cm, semithick}]
    \node[fill=gray!10] (empty) at (3,0) {\Large $\varnothing$};
    \node[dotted] (x1) at (0,-1.5) {$\{x_1\}$};
    \node[fill=gray!10] (x3) at (3,-1.5) {$\{x_3\}$};
    \node[fill=gray!10] (x4) at (6, -1.5) {$\{x_4\}$};
    \node[dotted] (x1x3) at (0,-3) {$\{x_1, x_3\}$};
    \node[dotted] (x1x4) at (3,-3) {$\{x_1, x_4\}$};
    \node[fill=gray!10] (x3x4) at (6,-3) {$\{x_3, x_4\}$};
    \node[dotted] (x1x3x4) at (3,-4.5) {$\{x_1, x_3, x_4\}$};

    \path[-latex, semithick]
    	(empty) edge[dotted, ->] (x1)
    	(empty) edge (x3)
    	(empty) edge (x4)
    	(x1) edge[dotted, ->] (x1x3)
    	(x1) edge[dotted, ->] (x1x4)
    	(x3) edge[dotted, ->] (x1x3)
    	(x3) edge (x3x4)
    	(x4) edge[dotted, ->] (x1x4)
    	(x4) edge (x3x4)
    	(x1x3) edge[dotted, ->] (x1x3x4)
    	(x1x4) edge[dotted, ->] (x1x3x4)
    	(x3x4) edge[dotted, ->] (x1x3x4);
\end{tikzpicture}
\end{minipage}
    \caption{Pruned parent graphs when incorprating the tiers $\mathbf{T} = \{\{x_1\}, \{x_2, x_3\}, \{x_4\}\}$ such that $x_1$ is a source variable and $x_4$ is a sink variable. Shown are the parent graphs $\mathcal{G}_{Pa}(x_4)$ for the sink variable $x_4$ (left), which in fact does not change, and $\mathcal{G}_{Pa}(x_2)$ (right), which has the same number of nodes as the pruned parent graph $\mathcal{G}_{Pa}(x_3)$. Note that the pruned parent graph $\mathcal{G}_{Pa}(x_1)$ only contains the empty set.}
    \label{fig:app_parent_graph_tiers}
\end{figure*}

\section{Additional experiment information}

\begin{figure*}
    \centering
    \includegraphics[width=0.91\linewidth]{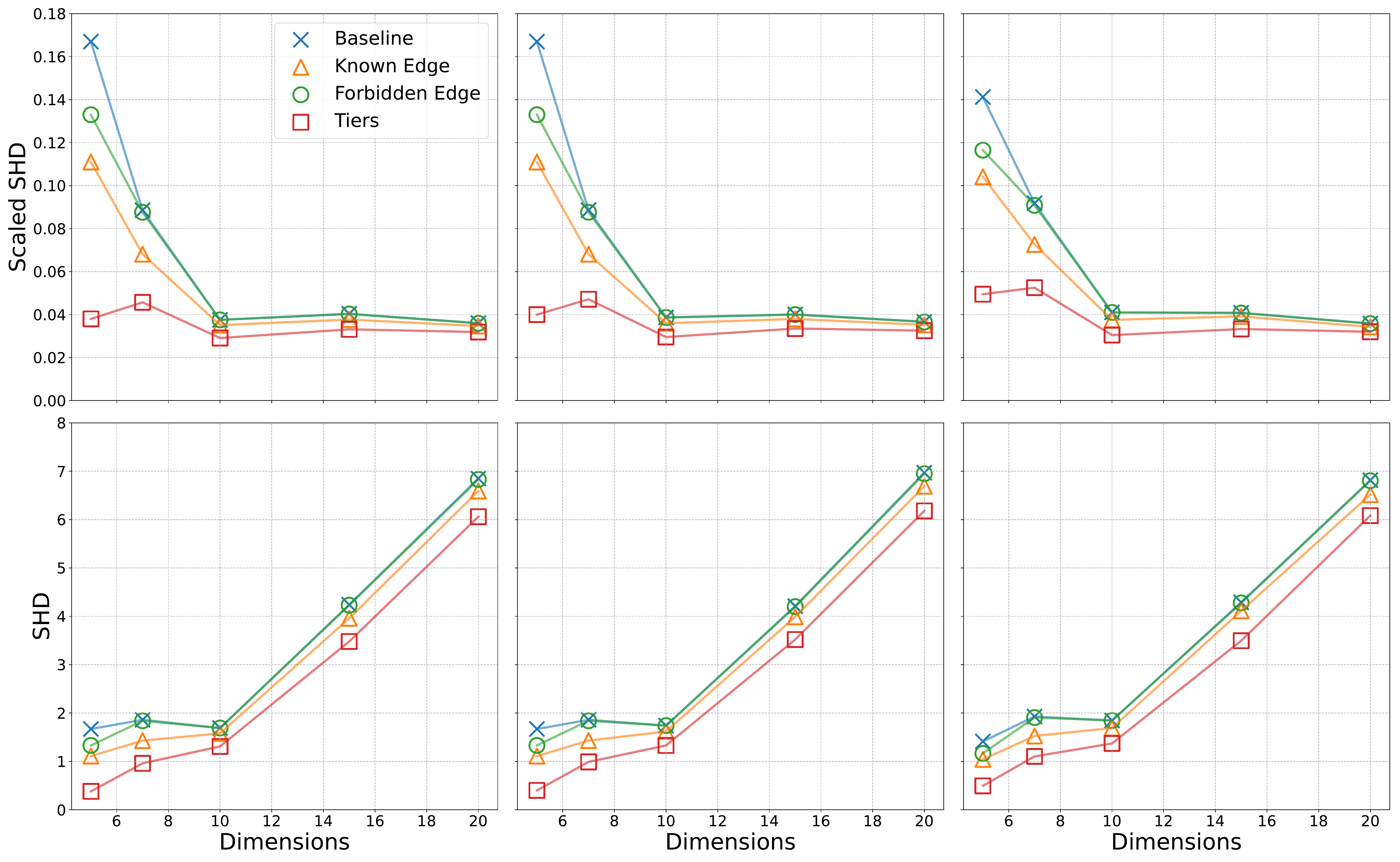}
    \caption{Scaled SHD (top row) and regular SHD (bottom row) for varying dimensions when applying regular {\astar} (left column), {\astar}-SuperStructure (middle column) and Local {\astar} (right column) on the synthetic data set, averaged over $100$ repeats.}
    \label{fig:synthetic_results_shd_app}
\end{figure*}

\begin{figure*}
    \centering
    \includegraphics[width=0.91\linewidth]{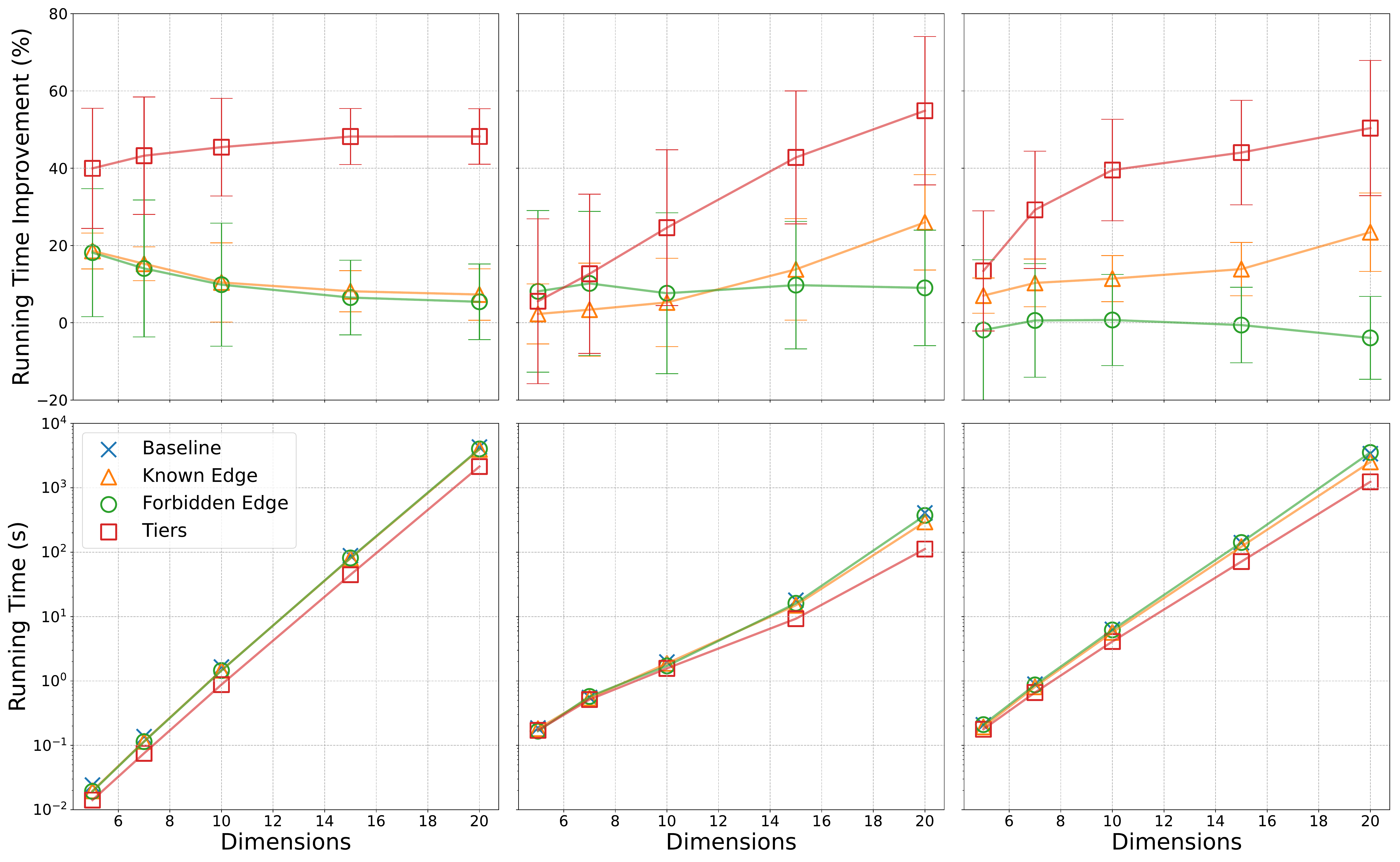}
    \caption{Percentage running time improvements (top row) and absolute running time (bottom row) for varying dimensions when applying regular {\astar} (left column), {\astar}-SuperStructure (middle column) and Local {\astar} (right column) on the synthetic data set, averaged over $100$ repeats. The error bars indicate one standard deviation from the mean.}
    \label{fig:synthetic_results_runtime_app}
\end{figure*}

\subsection{Evaluation}

As mentioned in the main text, we evaluate the performance of all causal discovery methods by their ability to recover the true graph, as measured by the structural hamming distance (SHD) between \emph{modified} CPDAGs. Recall that {\astar}-based methods return a \emph{modified} CPDAG, where Meek's rules have been applied post-discovery to ensure that the resulting CPDAG complies with the provided domain knowledge. To ensure a fair comparison, we therefore compute the CPDAG of the ground-truth DAG and then similarly apply Meek's rules. This allows us to measure causal discovery performance via the SHD between the modified CPDAG estimate and the modified CPDAG of the ground-truth. We consider the SHD between different types of edges in a CPDAG to be $1$ unless they are exactly the same, e.g.~$\text{SHD}(\leftarrow, \rightarrow) = 1$, $\text{SHD}(\leftarrow, \text{---}) = 1$ and $\text{SHD}(\leftarrow, \leftarrow) = 0$.
Furthermore, in order to allow for a fair comparison of running times, all experiments were run on a single AMD Ryzen 7 3800X CPU.

\subsection{Synthetic data}

\begin{table*}[!t]
\centering
\begin{tabular}{lccccc}  
\toprule
& \multicolumn{5}{c}{Dimensions} \\
\cmidrule(r){2-6}
 & 5 & 7 & 10 & 15 & 20 \\
\midrule
{\astar}                 & 0.0001 & 0.0003 & 0.0020 & 0.0003 & 0.0001 \\
{\astar}-SuperStructure  & 0.0001 & 0.0003 & 0.0012 & 0.0017 & 0.0001 \\
Local {\astar}           & 0.0060 & 0.0004 & 0.0017 & 0.0614 & 0.0001 \\
\bottomrule
\end{tabular}
\caption{P-values from a student's t-test with the null hypothesis that integrating \emph{known edges} improves SHD.}
\label{table1}
\end{table*}

\begin{table*}[!t]
\centering
\begin{tabular}{lccccc}  
\toprule
& \multicolumn{5}{c}{Dimensions} \\
\cmidrule(r){2-6}
 & 5 & 7 & 10 & 15 & 20 \\
\midrule
{\astar}                 & 0.0144 & 0.2085 & N/A & 0.1599 & 0.0792 \\
{\astar}-SuperStructure  & 0.0144 & 0.2085 & N/A & 0.1599 & 0.0792 \\
Local {\astar}           & 0.0343 & 0.2085 & N/A & 0.1599 & 0.1599 \\
\bottomrule
\end{tabular}
\caption{P-values from a student's t-test with the null hypothesis that integrating \emph{forbidden edges} improves SHD. N/A indicates that no student's t-test could be performed because no run showed any SHD improvements.}
\label{table2}
\end{table*}

\begin{table*}[!t]
\centering
\begin{tabular}{lccccc}  
\toprule
& \multicolumn{5}{c}{Dimensions} \\
\cmidrule(r){2-6}
 & 5 & 7 & 10 & 15 & 20 \\
 % & $(\times 10^{-4})$ & $(\times 10^{-4})$ & $(\times 10^{-4})$ & $(\times 10^{-4})$ & $(\times 10^{-4})$ \\
\midrule
{\astar}                 & $0.0007 \times 10^{-4}$ & $0.0182 \times 10^{-4}$ & $6.8020 \times 10^{-4}$ & $0.7873 \times 10^{-4}$ & $0.0040 \times 10^{-4}$ \\
{\astar}-SuperStructure  & $0.0006 \times 10^{-4}$ & $0.0356 \times 10^{-4}$ & $3.8586 \times 10^{-4}$ & $0.7552 \times 10^{-4}$ & $0.0010 \times 10^{-4}$ \\
Local {\astar}           & $0.3255 \times 10^{-4}$ & $0.2331 \times 10^{-4}$ & $1.1118 \times 10^{-4}$ & $0.2188 \times 10^{-4}$ & $0.0007 \times 10^{-4}$ \\
\bottomrule
\end{tabular}
\caption{P-values from a student's t-test with the null hypothesis that integrating \emph{tiers} improves SHD.}
\label{table3}
\end{table*}

In our synthetic data experiments, we generate random ground-truth DAGs $\mathcal{G}$ using the Erd\H{o}s-R\'{e}nyi algorithm, with an average degree of $2$ and varying dimensions of $p \in \{ 5, 7, 10, 15, 20\}$. We then generate data by means of a linear structural equation model (SEM), i.e.~
\begin{equation}
    x_i = \sum_j^{|Pa_{\mathcal{G}}(x_i)|} w_j x_j + \epsilon_i,
\end{equation}
where $x_j \in Pa_{\mathcal{G}}(x_i)$ is a parent of $x_i$ in $\mathcal{G}$. The linear weights $w_j$ are randomly sampled from either $U(-0.2, -0.8)$ or $U(0.2, 0.8)$, chosen by a coin flip. The noise variables $\epsilon_i$ are sampled from a Gaussian distribution $\epsilon_i \sim \mathcal{N}(0, \sigma_i)$, where $\sigma_i \sim U(1, 2)$. For each random DAG $\mathcal{G}$ we generate $500$ observations using the same set of linear weights and noise standard deviations $\{\sigma_i\}_{i=1}^p$.

We here provide further experimental results that supplement those in the main text. Specifically, Figure~\ref{fig:synthetic_results_shd_app} shows the scaled SHD (top row) and regular (SHD) for all types of domain knowledge, for all {\astar}-based methods and for varying dimensions. The SHD improvements of {\astar}-SuperStructure and Local {\astar} are largely similar to those of regular {\astar}, with Local {\astar} performing slightly worse. Similarly, Figure~\ref{fig:synthetic_results_runtime_app} shows additional figures for computational running time improvements. The top row shows the percentage improvements including error bars, corresponding to one standard deviation. The behavior of Local {\astar} is relatively similar to that of {\astar}-SuperStructure. Notably, however, for Local {\astar} the running times when integrating forbidden edges (shown in green) are worse than when not integrating any domain knowledge. The poor running time improvements when integrating forbidden edges with {\astar}-SuperStructure and Local {\astar}, as compared to integrating other types of domain knowledge, may be explained by the fact that graphical LASSO also prunes the parent graphs. Specifically, the pruning resulting from applying graphical LASSO has the same effect as integrating forbidden edges, since the resulting super-structure contains information about which edges might exist and, more importantly, which edges cannot exist. In this sense, integrating forbidden edges does not help with the running times of {\astar}-SuperStructure and Local {\astar} because graphical LASSO is likely to already have captured those forbidden edges as part of estimating the super-structure.

We note that all of our experiments were run on a single CPU and, unlike the other methods, Local {\astar}
has the potential to be parallelized over several CPUs~\citep{ng2021}, which means that it could potentially be faster when being parallelized. The construction of local clusters in Local {\astar} (which would be the aspect to parallelize) carries some computational overhead, which would explain its poor performance on a single CPU core. The bottom row in Figure~\ref{fig:synthetic_results_runtime_app} shows corresponding absolute running times in seconds. {\astar}-SuperStructure and Local {\astar} tend to be slower than regular {\astar} for small dimensions, presumably due to the computational overhead of the super-structure estimation, while they are faster for larger dimensions. {\astar}-SuperStructure seems to have exceptional running time performance for large dimensions, as compared to the other methods.

Lastly, we present p-values from student's t-tests with the null hypothesis that integrating known edges (Table~\ref{table1}), forbidden edges (Table~\ref{table2}) or tiers (Table~\ref{table3}) improves the SHD, as compared to not integrating any domain knowledge. Tables~\ref{table1} and Table~\ref{table3} show that we can be confident, with a p-value threshold of $0.05$, that integrating known edges and tiers improves the SHD for nearly all dimensions. However, Table~\ref{table2} shows that the same is not true for integrating forbidden edges, as we can only confidently say that this improves SHD for small, i.e.~$5$, dimensions.  

\subsection{Sachs data}

In Figure~\ref{fig:sachs_graph} we visualize the ground-truth DAG corresponding to the real, protein signal-processing data set of~\citet{sachs2005}. This ground-truth was formed using expert knowledge and, partly, interventional data. See~\citet{sachs2005} for more information.

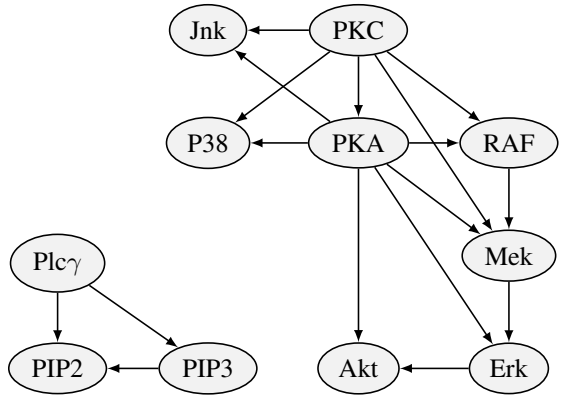
\begin{figure}[h!]
\centering
\begin{tikzpicture}[
    every node/.append style={draw, ellipse, minimum width = 1cm, minimum height = 0.5cm, semithick}]
    \node[fill=gray!10] (plcg) at (0,1.4) {Plc$\gamma$};
    \node[fill=gray!10] (pip2) at (0,0) {PIP2};
    \node[fill=gray!10] (pip3) at (2,0) {PIP3};
    \node[fill=gray!10] (pkc) at  (4,4.5) {PKC};
    \node[fill=gray!10] (jnk) at  (2,4.5) {Jnk};
    \node[fill=gray!10] (p38) at  (2,3) {P38};
    \node[fill=gray!10] (pka) at  (4,3) {PKA};
    \node[fill=gray!10] (raf) at  (6,3) {RAF};
    \node[fill=gray!10] (mek) at  (6,1.5) {Mek};
    \node[fill=gray!10] (erk) at  (6,0) {Erk};
    \node[fill=gray!10] (akt) at  (4,0) {Akt};

    \path[-latex, semithick]
    	(plcg) edge (pip2)
    	(plcg) edge (pip3)
    	(pip3) edge (pip2)
    	(pkc) edge (pka)
    	(pkc) edge (jnk)
    	(pkc) edge (p38)
    	(pka) edge (jnk)
    	(pka) edge (p38)
    	(pkc) edge (raf)
    	(pka) edge (raf)
    	(pka) edge (mek)
    	(pkc) edge (mek)
    	(pka) edge (erk)
    	(pka) edge (akt)
    	(raf) edge (mek)
    	(mek) edge (erk)
    	(erk) edge (akt);
\end{tikzpicture}
    \caption{Ground-truth DAG for the real protein signalling dataset of~\citet{sachs2005}, formed by experts using interventional data.}
    \label{fig:sachs_graph}
\end{figure}

\end{document}